\documentclass[twoside]{article}
\usepackage[accepted]{aistats2023}

\usepackage{algorithm}
\usepackage{algpseudocode}
\usepackage[accepted]{aistats2023}
\usepackage{aligned-overset}
\usepackage{amsmath}
\usepackage{amssymb}
\usepackage{amsthm}
\usepackage{bbm}
\usepackage{bm}
\usepackage{booktabs}
\usepackage{calrsfs}
\usepackage{centernot}
\usepackage{colonequals}
\usepackage{hyperref}
\usepackage{mdsymbol}
\usepackage{tikz}
\usetikzlibrary{arrows}
\usetikzlibrary{datavisualization.formats.functions}
\usetikzlibrary{decorations.pathreplacing,calligraphy}
\usepackage{pgfplots}
\pgfplotsset{compat=1.17}

\usepackage{microtype}
\usepackage{graphicx}
\usepackage{subcaption}
\usepackage{wrapfig}
\usepackage[rightcaption]{sidecap}
\usepackage{xfrac}

\usepackage[round]{natbib}

\begin{document}

\theoremstyle{plain}
\newtheorem{theorem}{Theorem}[section]
\newtheorem{proposition}[theorem]{Proposition}
\newtheorem{lemma}[theorem]{Lemma}
\newtheorem{corollary}[theorem]{Corollary}
\theoremstyle{definition}
\newtheorem{definition}[theorem]{Definition}
\newtheorem{assumption}[theorem]{Assumption}
\theoremstyle{remark}
\newtheorem{remark}[theorem]{Remark}
\newtheorem{example}[theorem]{Example}



\twocolumn[

\aistatstitle{Uncertainty Estimates of Predictions via a General Bias-Variance Decomposition}

\aistatsauthor{ Sebastian G. Gruber \And Florian Buettner }

\aistatsaddress{ German Cancer Research Center (DKFZ) \\
  German Cancer Consortium (DKTK) \\
  Goethe University Frankfurt, Germany \\
  \texttt{sebastian.gruber@dkfz.de} \\ \And  German Cancer Research Center (DKFZ) \\
    German Cancer Consortium (DKTK) \\
    Frankfurt Cancer Institute, Germany \\
    Goethe University Frankfurt, Germany \\
   \texttt{florian.buettner@dkfz.de} \\ } ]

\begin{abstract}
    Reliably estimating the uncertainty of a prediction throughout the model lifecycle is crucial in many safety-critical applications.
    The most common way to measure this uncertainty is via the predicted confidence.
    While this tends to work well for in-domain samples, these estimates are unreliable under domain drift and restricted to classification.
    Alternatively, proper scores can be used for most predictive tasks but a bias-variance decomposition for model uncertainty does not exist in the current literature.
    In this work we introduce a general bias-variance decomposition for strictly proper scores, giving rise to the Bregman Information as the variance term.
    We discover how exponential families and the classification log-likelihood are special cases and provide novel formulations.
    Surprisingly, we can express the classification case purely in the logit space.
    We showcase the practical relevance of this decomposition on several downstream tasks, including model ensembles and confidence regions.
    Further, we demonstrate how different approximations of the instance-level Bregman Information allow out-of-distribution detection for all degrees of domain drift.
\end{abstract}

\section{INTRODUCTION}

A core principle behind the success of modern Machine and Deep Learning approaches are loss functions, which are used to optimize and compare the goodness-of-fit of predictive models. Typical loss functions, such as the Brier score or the negative log-likelihood, capture not only predictive power (in the sense of accuracy) but also predictive uncertainty. The latter is particularly relevant in sensitive forecasting domains, such as cancer diagnostics \citep{HAGGENMULLER2021202}, genotype-based disease prediction \citep{KatsaouniTashkandiWieseSchulz+2021+871+885} or climate prediction \citep{yen2019application}. \\
Proper scores are a common occurrence as loss functions for probabilistic modelling since their defining criterion is to assign the best value to the target distribution as prediction.
Consequently, they are widely applicable from quantile regression \citep{gneitingscores} to generative models \citep{song2021scorebased}.
They are a generalization of the log-likelihood and also cover exponential families \citep{grunwald2004game}.
However, for such loss functions, it is not clear how we can decompose them such that a component capturing predictive uncertainty arises.
Consequently, predictive uncertainty is typically only considered as variance of predictions or, in classification, via the confidence score associated to the top-label prediction.
Such confidence scores capture the predictive uncertainty well if they are calibrated, namely if the confidence of a prediction matches its true likelihood \citep{guo2017calibration}.
However, the calibration error of these confidence scores typically increases under domain drift, making them an unreliable measure for predictive uncertainty in many real-world applications \citep{ovadia2019can, tomani2021towards}. \\
In this work, we discover the Bregman Information as a natural replacement of model variance via a bias-variance decomposition for strictly proper scores.
The Bregman Information generalizes the variance of a random variable via a closed-form definition based on a generating function \citep{banerjee2005clustering}.
In the case of our decomposition, this generating function is a convex conjugate directly associated with the respective proper score.
The source code for the experiments is openly accessible at \url{https://github.com/MLO-lab/Uncertainty_Estimates_via_BVD}.
We summarize our \textbf{contributions} in the following:
\begin{itemize}
    \item In Section \ref{sec:gbvd}, we generalize relevant properties to functional Bregman divergences, which allows for deriving a bias-variance decomposition for strictly proper scores.
    Via Bregman Information, we give novel formulations for decompositions of exponential families and the classification log-likelihood in the logit space.
    \item We generalize the law of total variance to Bregman Information and show how ensemble predictions marginalize out a specific source of uncertainty in Section \ref{sec:ens}. 
    We also propose a general way to give confidence regions for predictions in Section \ref{sec:CI}.
    \item We showcase experiments on how typical classifiers differ in their Bregman Information in Section \ref{sec:exp}.
    There, we demonstrate that the Bregman Information can be a more meaningful measure of out-of-domain uncertainty compared to the confidence score in the case of corrupted CIFAR-10 and ImageNet (Figure \ref{fig:ood_cif10} and Algorithm \ref{alg:BI}).
\end{itemize}

\begin{figure}[t]
\vspace{.1in}
\centering
\includegraphics[width=\columnwidth]{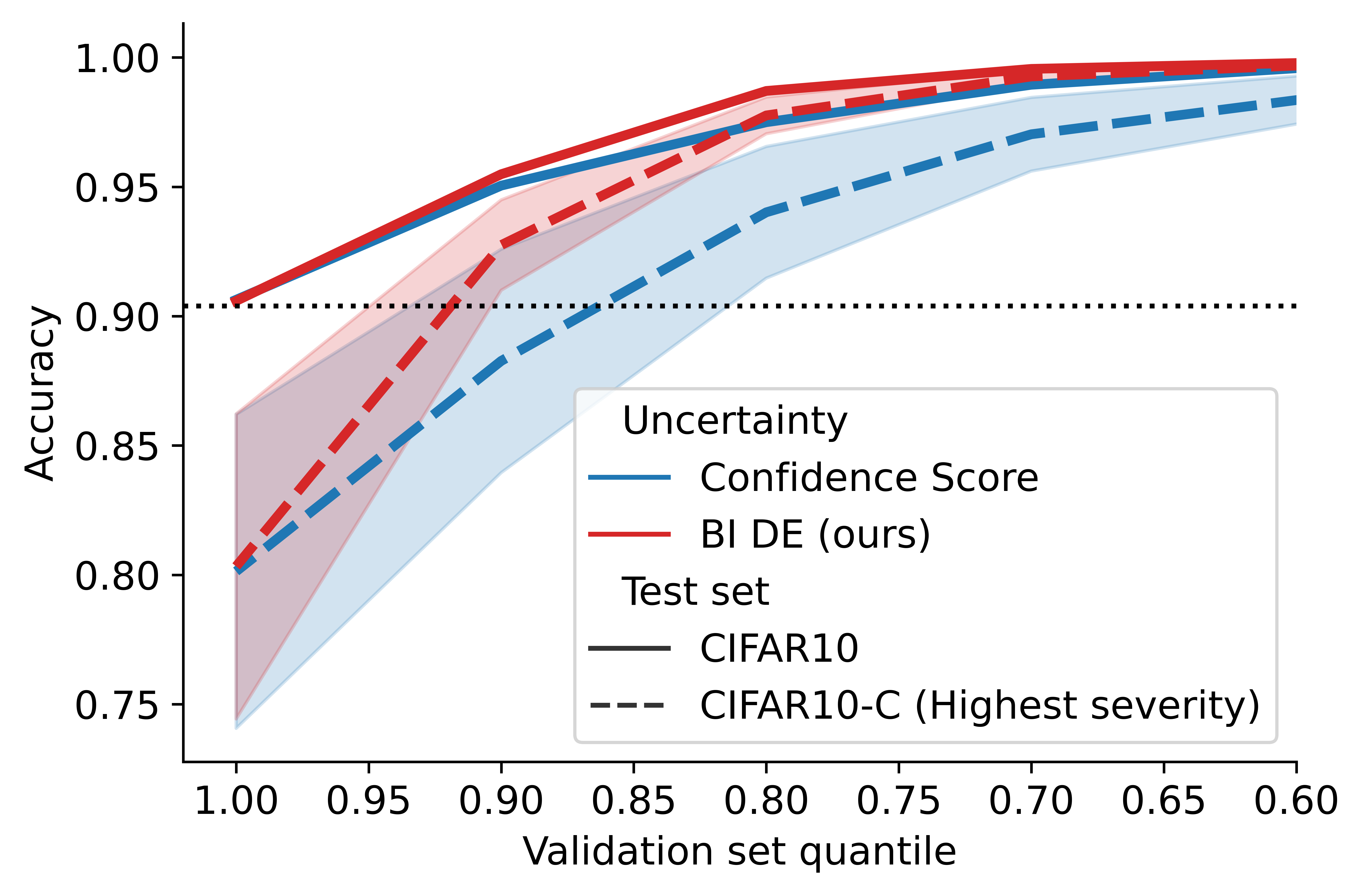}
\vspace{.1in}
\caption{
    Accuracy after discarding test instances with high levels of uncertainty. We can discard fewer samples to reach better accuracy when using Bregman Information as uncertainty measure. For example, to achieve the validation set accuracy (dotted line) for severely corrupted data we only have to discard $\sim$7\% of most uncertain in-domain samples contrary to $\sim$14\% when using the confidence score.
    The standard deviation bounds stem from different types of corruption.
}
\label{fig:ood_cif10}
\end{figure}

%\section{PRELIMINARIES AND RELATED WORK}
\section{BACKGROUND}
\label{sec:prel}

In this section, we first start with a basic introduction of Bregman divergences and Bregman Information.
We specifically mention recent developments for functional Bregman divergences as we will require and provide generalizations to this topic.
Then follows another introduction into the basic concepts of proper scores and exponential families, which are related to Bregman divergences.
Finally, we will discuss other proposed bias-variance decompositions in the literature to put our contribution into perspective.

\subsection{Bregman Divergences and Bregman Information}
\label{sec:every_breg}

Bregman divergences are a class of divergences occurring in a wide range of applications \citep{bregman1967relaxation, banerjee2005clustering, https://doi.org/10.48550/arxiv.cs/0611123, si2009bregman, gupta2022ensembles}.
We use the following definition.

\begin{definition}[\cite{bregman1967relaxation}]
    Let $\phi \colon U \to \mathbb{R}$ be a differentiable, convex function with $U \subset \mathbb{R}^d$.
    The \textbf{Bregman divergence} generated by $\phi$ of $x,y \in U$ is defined as
    \begin{equation*}
        d_\phi \left( x, y \right) = \phi \left( y \right) - \phi \left( x \right) - \left\langle \nabla \phi \left( x \right), y - x \right\rangle.
    \end{equation*}
\label{def:BD}
\end{definition}

It can be interpreted geometrically as the difference between $\phi$ and the supporting hyperplane of $\phi \left( x \right)$ at $y$.
We have $d_\phi \left( x, y \right) \geq 0$ with $d_\phi \left( x, y \right) = 0$ if $x=y$. \\
By definition, we can use Bregman divergences for scalar and vector inputs.
But, in the infinite-dimensional case, for example when dealing with a continuous distribution space $\mathcal{P}$, the gradient vector and the inner product are not defined anymore.
As a solution to this, \cite{https://doi.org/10.48550/arxiv.cs/0611123} introduce functional Bregman divergences by replacing the inner product term with the Fr\'echet derivative.
The authors showed that the functional case generalizes the standard case.
Since some relevant functions are not Fr\'echet differentiable, \cite{10.3150/16-BEJ857} offers an alternative approach to define the functional case based on subgradients.
In the context of dual vector spaces with pairing $\left\langle .,. \right\rangle$, a subgradient $x^\prime$ at point $x \in U$ of a convex function $\phi \colon U \to \mathbb{R}$ fulfills the property $\phi \left( y \right) \geq \phi \left( x \right) + \left\langle x^\prime, y - x \right\rangle$ for all $y \in U$.
A function $\phi^\prime \left( x \right)$ which maps to a subgradient of $\phi \left( x \right)$ for all $x \in U$ is called a selection of subgradients, or, if it is unambiguous in the context, just \textbf{subgradient} of $\phi$.
In general, subgradients are not unique, unlike gradients.
\cite{10.3150/16-BEJ857} proposes to use the vector space $\mathcal{L} \left( \mathcal{P} \right)$ of $\mathcal{P}$-integrable functions and the vector space $\mathrm{span} \mathcal{P}$ of finite linear combinations of elements from $\mathcal{P}$.
These spaces are dual with the pairing "$\cdot$" defined as $f \cdot P = \int f \mathrm{d} P$ for $f \in \mathcal{L} \left( \mathcal{P} \right)$ and $P \in \mathrm{span} \mathcal{P}$.
%This operator is similar to the inner product for categorical distributions.
%Further, they use the definition of subgradients:
%Here, $U$ is a subset of a vector space with pairing $\left\langle .,. \right\rangle$ to a dual vector space.
They proceed to define the by $\left(\phi, \phi^\prime \right)$ generated \textbf{functional Bregman divergence} as $d_{\phi, \phi^\prime} \left( x, y \right) = \phi \left( y \right) - \phi \left( x \right) - \phi^\prime \left( x \right) \cdot \left( y - x \right)$.
We will also use this definition for general vector spaces as long as a subgradient is defined.
In Section \ref{sec:gbvd}, we encounter the case when a subgradient $\phi^\prime_r$ of $\phi$ is only defined on a smaller domain $V \subset U$.
Then, we refer to $d_{\phi, \phi^\prime_r} \colon V \times U \to \mathbb{R}$ as a \textbf{restricted} functional Bregman divergence. \\
%We will use these definitions for a general bias-variance decomposition of proper scores, which we introduce in Section \ref{sec:gbvd}. \\
The \textbf{convex conjugate} $\phi^*$ of a function $\phi$ is defined as $\phi^* \left( x^* \right) = \sup_y \left\langle x^*, y \right\rangle - \phi \left( y \right)$ in the context of dual vector spaces \citep{zalinescu2002convex}.
If $\phi$ is differentiable and strictly convex, then 
$\left(\nabla \phi \right)^{-1} = \nabla \phi^*$ and $\phi^{**} = \phi$.
For this case, \cite{banerjee2005clustering} give the important fact that $d_\phi \left( x, y \right) = d_{\phi^*} \left( \nabla \phi \left( y \right), \nabla \phi \left( x \right)\right)$.
That is, by using the convex conjugate, we can flip the arguments in a Bregman divergence.
To derive our main contribution, we will state a generalization of this property to functional Bregman divergences in Lemma \ref{le:conj_breg}.

\begin{figure}
\vspace{.2in}
    \centering
    \resizebox{0.9\columnwidth}{!}{%
    \begin{tikzpicture}
        \datavisualization [school book axes,
        visualize as smooth line/.list={one},
        y axis={label={}, ticks={step=2}},
        x axis={label={$x$}},
        one={style={red, thick}}]
        data [set=one,format=function] {
            var x : interval [-3:4];
            func y = ln(1 + exp(\value x) );
        };
        
        % X distribution
        %\node[font=\small] at (-1.8, 3) {$X \sim \mathcal{U} \left( \{ -2, 3 \} \right)$};

        \node at (2.7, 4) {$\ln \left( 1 + e^x \right)$};
        \draw [thick] (-2, 0.1269) -- (3, 3.0486);
        \draw [dashed, thick] (0.5, 1.5878) -- (3.3, 1.5878);
        \draw [dashed, thick] (0.5, 0.9741) -- (3.3, 0.9741);
        \draw [dashed, thick] (0.5, 1.5878) -- (0.5, -0.2);
        
        %\node[font=\small] at (2.7, 1.8) {$\mathbb{E} \left[ \mathrm{LSE} \left( X \right) \right]$};
        %\node[font=\small] at (2.7, 0.7) {$ \mathrm{LSE} \left( \mathbb{E} \left[ X \right] \right)$};
        \node[inner sep=1pt,] at (0.5, -0.4) {$\mathbb{E} \left[ X \right]$};
        \node at (4.2, 1.29) {$\Big\} \; \mathbb{B}_{\sigma_+} \left[ X \right]$};
        
    \end{tikzpicture}
    }
\vspace{.2in}
\caption{Illustration of the Bregman Information generated by the softplus function $\sigma_+ \left(x \right) = \ln \left(1 + e^x \right)$ of a binary random variable.
}
\label{fig:B_LSE}
\end{figure}

We can also use Bregman divergences to quantify the variability or deviation of a random variable.
Throughout this work, the following definition is a central concept.

\begin{definition}[\cite{banerjee2005clustering}]
    Let $\phi \colon U \to \mathbb{R}$ be a differentiable, convex function.
    The \textbf{Bregman Information} (generated by $\phi$) of a random variable $X$ with realizations in $U$ is defined as
    \begin{equation*}
        \mathbb{B}_\phi \left[ X \right] = \mathbb{E} \left[ d_\phi \left( \mathbb{E} \left[ X \right], X \right) \right].
    \end{equation*}
\label{def:BI}
\end{definition}

The Bregman Information generalizes the variance of a random variable since both are equal if we set $U = \mathbb{R}$ and $\phi \left( x \right) = x^2$.
Thus, one interpretation of the Bregman Information is that it measures the divergence of a random variable from its mean.
Another representation, which does not depend on $d_\phi$, is $\mathbb{B}_\phi \left[ X \right] = \mathbb{E} \left[ \phi \left( X \right) \right] - \phi \left( \mathbb{E} \left[ X \right] \right)$.
\cite{banerjee2005clustering} show that this follows from the original definition.
Recall that Jensen's inequality gives $\mathbb{E} \left[ \phi \left( X \right) \right] \geq \phi \left( \mathbb{E} \left[ X \right] \right)$.
Consequently, a second interpretation of the Bregman Information is that it measures the gap between both sides of Jensen's inequality of the convex function $\phi$ and random variable $X$ \citep{banerjee2005clustering}.
%It also shows that a generalization to the functional case is invariant towards the chosen subgradient of the underlying functional Bregman divergence.
It also shows that we do not require a subgradient for a generalization to the functional case.
Thus, we define the \textbf{functional Bregman Information} generated by a non-differentiable convex $\phi$ as $\mathbb{B}_\phi \left[ X \right] \coloneqq \mathbb{E} \left[ \phi \left( X \right) \right] - \phi \left( \mathbb{E} \left[ X \right] \right)$.
The Bregman Information generated by the \emph{softplus} function is depicted in Figure \ref{fig:B_LSE} for a binary random variable.
The softplus finds use as an activation function in neural networks \citep{glorot2011deep, pml1Book}.
Its generalization is the so-called LogSumExp function (c.f. Section \ref{sec:cll}). \\
In Section \ref{sec:gbvd}, the Bregman Information will play a critical role in our bias-variance decompositions since it represents the variance term.
Further, the LogSumExp-generated version covers the variance term for classification.
It reduces to the softplus version for the binary case.

\subsection{Proper Scores and Exponential Families}

\cite{gneitingscores} give an extensive and approachable overview of proper scores.
In short, proper scores put a negative loss on a distribution prediction $P \in \mathcal{P}$ for a target random variable $Y \sim Q \in \mathcal{P}$ and reach their maximum if $P = Q$.
For a concise statement of our main result, we require a more technical definition provided in the following similar to \cite{hendrickson1971proper}, \cite{ovcharov2015existence}, and \cite{10.3150/16-BEJ857}.
We call a function $S \colon \mathcal{P} \to \mathcal{L} \left( \mathcal{P} \right)$ \textbf{scoring rule} or just \textbf{score}.
Note that for a given $P$, $S \left( P \right)$ maps into a function space and can be again evaluated on an observation $y$, like $S \left( P \right) \left( y \right)$.
To assess the goodness-of-fit between distributions $P$ and $Q$, we use the expected score 
$S \left(P \right) \cdot Q = \mathbb{E}_{Y \sim Q} \left[ S \left(P \right) \left( Y \right) \right]$.
A score is defined to be \textbf{proper} on $\mathcal{P}$ if and only if
$S \left(P \right) \cdot Q \leq S \left(Q \right) \cdot Q$
holds for all $P, Q \in \mathcal{P}$,
and \textbf{strictly proper} if and only if an equality implies $P = Q$.
In other words, a score is proper if predicting the target distribution gives the best expectation and strictly proper if no other prediction can achieve this value.
Note that the choice of $\mathcal{P}$ is relevant: The negative squared error of a mean prediction is strictly proper for normal distributions with fixed variance but only proper if the variance varies.
Given a proper score, the associated \textbf{negative entropy} $G \colon \mathcal{P} \to \mathbb{R}$ is defined as
$G \left( Q \right) = S \left(Q \right) \cdot Q$.
It represents the highest reachable value for a given target.
If $\mathcal{P}$ is convex, the negative entropy has $S$ as a subgradient and is (strictly) convex if and only if $S$ is (strictly) proper.
For this case, \cite{10.3150/16-BEJ857} proved that a proper score is closely related to a functional Bregman divergence generated by the associated negative entropy via $G \left( Q \right) - S \left(P \right) \cdot Q = d_{G, S} \left( P, Q \right)$.
An example of such a relation is the Kullback-Leibler divergence and the Shannon entropy associated with the log score (log-likelihood).

\begin{table*}
  \caption{
  Examples of exponential families.
  The mapping defines the relation between natural parameters and typical parameters.
  We denote the dummy-encoding for a $x \in \left\{1, \dots, k \right\}$ with $\mathrm{d}_x$, class probabilities with $p_j$, mean with $\mu$, and standard deviation with $\sigma$.
  The Bernoulli distribution is a special case of the categorical distribution for $k=2$.
  }

  \label{tab:expfam}
  \centering
  \begin{tabular}{lllllll}
    \toprule
    Distribution              & $\mathcal{T}$ & $\Theta$ & $T \left( x \right)$ & $A \left( \theta \right)$ & $h \left( x \right)$ & Mapping \\
    \midrule
    Categorical ($k$-classes) & $\left\{1, \dots, k \right\}$ & $\mathbb{R}^{k-1}$ & $\mathrm{d}_x$ & $\ln \left( 1 + \sum_{i=1}^{k-1} \exp \theta_i \right)$ & $1$ & $p_j = \frac{\exp \theta_j}{1 + \sum_{i=1}^{k-1} \exp \theta_i}$ \\
    Normal (known $\sigma$) & $\mathbb{R}$ & $\mathbb{R}$ & $x$ & $\frac{\theta^2}{2}$ & $\frac{\exp{ \left( \frac{-x^2}{2 \sigma^2} \right)}}{\sqrt{2 \pi} \sigma}$ & $\mu = \theta \sigma$ \\
    %Normal (unknown $\sigma$) & $\mathbb{R}$ & $\mathbb{R} \times \mathbb{R}_{<0}$ & $\left(x, x^2 \right)^\intercal$ & $\frac{- \theta_1^2}{4 \theta_2} - \frac{1}{2} \log \left( -2 \theta_2 \right)$ & $\frac{1}{\sqrt{2 \pi}}$ \\
    \bottomrule
  \end{tabular}
\end{table*}

Next, we summarize relevant aspects of exponential families.
\cite{banerjee2005clustering} provides a more extensive introduction.
For a support set $\mathcal{T}$, the probability density/mass function $p_\theta$ at a point $x \in \mathcal{T}$ of an \textbf{exponential family} is given by $p_\theta \left( x \right) = \exp{ \left( \left\langle \theta, T \left( x \right) \right\rangle - A \left( \theta \right) \right)} h \left( x \right)$.
Here, we call $\theta \in \Theta$ the natural parameter of the convex parameter space $\Theta \subset \mathbb{R}^d$, $T$ is the sufficient statistic, and $A$ is the log-partition.
Table \ref{tab:expfam} gives two relevant examples and shows the mapping between typical and natural parameters.
Further examples are the Dirichlet, exponential, and Poisson distributions.
There are two relevant properties which we will require for our results.
One is that $A$ is a strictly convex function.
The other is $\mathbb{E} \left[ T \left( X \right) \right] = \nabla A \left( \theta \right)$ for $X \sim p_\theta$.
\cite{banerjee2005clustering} proved under mild conditions that an exponential family relates to a Bregman divergence and vice versa via the negative log-likelihood.
\cite{grunwald2004game} proved a similar link between proper scores and exponential families.

As we can see, exponential families, proper scores, and Bregman divergences have strong relationships to one another.
By generalizing some properties to the functional case, this relationship will allow us to state every variance term as (functional) Bregman Information.
%Unfortunately, the literature addresses mostly only vector-based Bregman divergences.
%One contribution of this work is to generalize the necessary properties to the functional case.
%Then, we can derive a general bias-variance decomposition of proper scores, where we will discover the functional Bregman Information as representative of the variance term.
In the case of the log-likelihood of an exponential family, the functional Bregman Information reduces to a vector-based Bregman Information.

\subsection{Other Bias-Variance Decompositions}

In general, all general decompositions in current literature are either for categorical, real-valued, or parametric predictions, and it is not clear if a decomposition for proper scores of non-parametric distributions is possible. \\
\cite{james1997generalizations} formulate a decomposition for any loss function of categorical or real-valued predictions but do not provide a closed-form solution for a given case.
\cite{domingos2000unified} introduce how a general bias-variance decomposition should look, though they stated it is unclear when or if this decomposition holds for a loss function.
\cite{james2003variance} provide a bias-variance decomposition for symmetric loss functions.
\cite{heskes1998bias} use the bias-variance decompositions for the Kullback-Leibler divergence, which allows to derive a decomposition for exponential families.
\cite{hansen2000general} proves that a bias-variance decomposition of a parametric prediction is only possible if the prediction belongs to an exponential family.
Importantly, they only introduce the specific decomposition for a given exponential family.
The decomposition is not formulated for the natural parameters and relies on the canonical link function.
Consequently, a relation to Bregman divergences and Bregman Information is missing, which we will provide. \\
A Pythagorean-like theorem for vector-based Bregman divergences is a known fact in literature \citep{jones1990general, csiszar1991least, della2002duality, dawid2007geometry, telgarsky2012agglomerative}.
An equality in this theorem implies a decomposition in the form of $\mathbb{E} \left[ d_\phi \left( X, y \right) \right] = \mathbb{E} \left[ d_\phi \left( X, x^* \right) \right] + d_\phi \left( x^*, y \right)$ with $x^* = \arg\min_z \mathbb{E} \left[ d_\phi \left( X, z \right) \right]$ \citep{pfau2013generalized}.
\cite{brofos2019bias}, \cite{brinda2019holder}, and \cite{yang2020rethinking} relate the classification log-likelihood to the Kullback-Leibler divergence and provide a bias-variance decomposition, where $X$ takes the form of a predictive probability vector.
They set $x^* \propto \exp \mathbb{E} \left[ \log X \right]$.
Note that predictions in the logit space require normalization to the log space, which will not be the case in our formulation.
\cite{gupta2022ensembles} build upon the Bregman divergence decomposition and use the notion of primal and dual space of the variance.
Even though the definitions are similar, the authors did not state the relation between Bregman Information and dual variance, for which they introduce a general law of total variance.

Due to the restriction of Bregman divergences to vector inputs, it is not clear if a decomposition for proper scores of non-parametric distributions is possible.
In other words, we require an extension of the current literature to functional Bregman divergences for a positive result.
In the following section, we provide the required generalization and unify the variance term in previous literature via the Bregman Information.

\section{A GENERAL BIAS-VARIANCE DECOMPOSITION}
\label{sec:gbvd}

In this section, we offer a general bias-variance decomposition for strictly proper scores.
The only assumptions are that the distribution set $\mathcal{P}$ is convex, the associated negative entropy is lower semicontinuous, and each respective expectation exists.
Further, we will discover that the variance term is the Bregman Information generated by the convex conjugate of the associated negative entropy.
This discovery generalizes and unifies decompositions in current literature for which exists a concrete form \citep{hansen2000general}, but also provides a closed formulation contrary to other general bias-variance decompositions \citep{james1997generalizations}.
All technical details and proofs are presented in Appendix \ref{app:proofs}.

\subsection{Functional Bregman Divergences of Convex Conjugates}
\label{sec:conj_breg}

The essential part for deriving our main result is the exchange of arguments in a functional Bregman divergence.
Note that a subgradient of a strictly convex function is injective.
Thus, its inverse exists on an appropriate domain, and this inverse is again a subgradient of the convex conjugate.
With that in mind, we can state the following.

\begin{lemma}
    Assume a strictly convex, lower semicontinuous function $G \colon \mathcal{P} \to \mathbb{R}$ has a subgradient $S$.
    Then, $d_{G^*, S^{-1}}$ is a restricted functional Bregman divergence
    with the properties
    \begin{itemize}
        \item $d_{G,S} \left( p, q \right) = d_{G^*, S^{-1}} \left( S \left( q \right), S \left( p \right) \right)$, and
        \item $d_{G^*, S^{-1}} \left( p^\prime, q^\prime \right) = d_{G,S} \left( S^{-1} \left( q^\prime \right), S^{-1} \left( p^\prime \right) \right)$.
    \end{itemize}
    For an appropriate random variable $Q^\prime$ we have
    \begin{equation*}
        \mathbb{E} \left[ d_{G^*, S^{-1}} \left( p^\prime, Q^\prime \right) \right] = \mathbb{B}_{G^*} \left[ Q^\prime \right] + d_{G^*, S^{-1}} \left( p^\prime, \mathbb{E} \left[ Q^\prime \right] \right).
    \end{equation*}
%    If there exists a subgradient $\bar{S}$ of $G^*$ of $S^{-1}$ can be extended to be a subgradient of ,
\label{le:conj_breg}
\end{lemma}

The last property is a generalization of the decomposition of Bregman divergences combined with the Definition of Bregman Information.
%Note that we can only define $d_{G^*, \bar{S}}$ on the range of $S$ for its first argument.
%The second argument can be any linear combination of subgradients at points of $G$.
%We present the exact technical details with the proof in Appendix \ref{app:proofs}. \\
Lemma \ref{le:conj_breg} confirms that a critical property of Bregman divergences is also well-defined for functional Bregman divergences.
Namely, we can exchange the arguments by changing the generating convex function to the convex conjugate in a dual space.
This insight leads us now to the main theoretical contribution of this work.

\subsection{A Decomposition for Proper Scores}

In the following, we present a bias-variance decomposition for strictly proper scores.
%Surprisingly, we will not have to make any additional assumptions, neither about the score or its entropy being differentiable nor about the set of predictions being convex.
Note that we require no assumptions about the score or its entropy being differentiable.

\begin{theorem}
For a strictly proper score $S$ with associated lower semicontinuous negative entropy $G$, an estimated prediction $\hat{f}$, and the target $Y \sim Q$, we have

\begin{equation*}
\begin{split}
%& \underbrace{\mathbb{E} \left[ S ( \hat{f} ) ( Y ) \right]}_{\text{error}} \! = \! \underbrace{- G \left( Q \right)}_{\text{noise}} \! + \! \underbrace{ \mathbb{B}_{G^*} \left[ S ( \hat{f} ) \right]}_{\text{"variance"}} \! + \! \underbrace{ d_{G^*} \left( S \left( Q \right), \mathbb{E} \left[ S ( \hat{f} ) \right] \right)}_{\text{bias}}.
& \underbrace{\mathbb{E} \left[ - S ( \hat{f} ) ( Y ) \right]}_{\text{Error}} = \\
& \quad \underbrace{- G \left( Q \right)}_{\text{Noise}} + 
\underbrace{ \mathbb{B}_{G^*} \left[ S ( \hat{f} ) \right]}_{\text{"Variance"}} + \underbrace{ d_{G^*, S^{-1}} \left( S \left( Q \right), \mathbb{E} \left[ S ( \hat{f} ) \right] \right)}_{\text{Bias}}.
\end{split}
\end{equation*}
%where $\bar{S}$ is an extension of $S^{-1}$.
\label{th:scores_bvd}
\end{theorem}

As we can see, the variance term is always represented by the Bregman Information generated by the convex conjugate of the associated negative entropy.
The theorem directly follows from the relationship between proper scores and functional Bregman divergences provided by \cite{10.3150/16-BEJ857} in combination with Lemma \ref{le:conj_breg} (c.f. Appendix \ref{app:proofs}).
We provide an example in the form of the prominent log score.
For conciseness, we denote all distributions with their respective densities.

\begin{example}
    The log score $S = \ln$ is strictly proper on any set of densities. 
    Its negative entropy is the negative Shannon entropy $H \left( p \right) = \int p \left( x \right) \ln p \left( x \right) \mathrm{d}x$.
    The convex conjugate is the log partition function $H^* \left( p^* \right) = \ln \int e^{p^* \left( x \right)} \mathrm{d}x$.
    Since $\mathbb{E} \left[ H^* \left( \ln \hat{f} \right) \right] = 0$, we receive the Bregman Information in the form $\mathbb{B}_{H^*} \left[ S \left( \hat{f} \right) \right] = - H^* \left( \mathbb{E} \left[ \ln \hat{f} \right] \right)$.
\label{ex:log}
\end{example}

Dealing with non-parametric distributions can be challenging both in theory and practice.
Consequently, we provide the following restriction to exponential families and then express neatly the decomposition through the natural parameters.
Specifically, $\mathbb{B}_{H^*}$ will be reduced to a vector-based Bregman Information.

\subsection{Exponential Families as a Special Case}
\label{sec:exp_decomp}

When we deal with densities, it is often more practical to consider parametric densities since they can be represented by a parameter vector instead of a function.
Particularly relevant classes are exponential families for which one uses the log-likelihood to assess the goodness of fit.
In the last section, we derived the decomposition for the log-likelihood expressed in the function space in Example \ref{ex:log}.
Thus, we provide the following special case as a novel formulation of \citep{hansen2000general}.

\begin{proposition}
For an exponential family density/mass function $p_{\hat{\theta}}$ with estimated natural parameter vector $\hat{\theta}$, log-partition $A$, and reference function $h$, we have the log likelihood decomposition

\begin{equation*}
%\begin{split}
 \underbrace{\mathbb{E} \left[ - \ln p_{\hat{\theta}} \left( Y \right) \right]}_{\text{NLL}} = \underbrace{n \left( \theta \right)}_{\text{Noise}} + \underbrace{ \mathbb{B}_{A} \left[ \hat{\theta} \right]}_{\text{"Variance"}} + \underbrace{ d_{A} \left( \theta, \mathbb{E} \left[ \hat{\theta} \right] \right)}_{\text{Bias}}
%\end{split}
\end{equation*}
where $n \left( \theta \right) = - A^* \left( \nabla A \left( \theta \right) \right) - \mathbb{E} \left[ \ln h \left( Y \right) \right]$ and $Y \sim p_\theta$.
\label{prop:exp_fam_decomp}
\end{proposition}

The proof in Appendix \ref{app:proofs} uses $\theta = \nabla A^* \left( \mathbb{E} \left[ T \left( Y \right) \right] \right)$ in the last step, where $T$ is the sufficient statistic.
This is a fundamental property of exponential families but only holds if the distribution assumption holds with the data.
Since this is usually not the case in practical scenarios, it is important to note that the decomposition still holds if we replace every $\theta$ with $\nabla A^* \left( \mathbb{E} \left[ T \left( Y \right) \right] \right)$.

\begin{example}
We can recover the textbook decomposition of the MSE by using $Y \sim \mathcal{N} \left( \mu, \sigma^2 \right)$ with unknown mean $\mu$ and known variance $\sigma^2$.
The necessary information is provided in Table \ref{tab:expfam}.
%Then, $A \left( \theta \right) = \frac{\theta^2}{2}$, $\theta = \frac{\mu}{\sigma}$, and $h \left( x \right) = \frac{1}{\sqrt{2\pi}\sigma}\exp{\frac{-x^2}{2 \sigma^2}}$, 
For the LHS in Proposition \ref{prop:exp_fam_decomp} we have $\mathbb{E} \left[ - \ln p_{\hat{\theta}} \left( Y \right) \right] = \frac{\mathbb{E} \left[ \left( Y - \hat{\mu} \right)^2 \right]}{2 \sigma^2} + \ln \sqrt{2 \pi}\sigma$.
And for the RHS, we get $n \left( \theta \right) = \frac{1}{2} + \ln \sqrt{2 \pi}\sigma$, $\mathbb{B}_{A} \left[ \hat{\theta} \right] = \frac{1}{2\sigma^2} \mathbb{V} \left[ \hat{\mu} \right]$, and $d_{A} \left( \theta, \mathbb{E} \left[ \hat{\theta} \right] \right) = \frac{\left(\mu - \mathbb{E} \left[ \hat{\mu} \right] \right)^2}{2 \sigma^2}$.
Finally, we subtract $\ln \sqrt{2 \pi}\sigma$ on both sides and then multiply with $2 \sigma^2$, resulting in
\begin{equation*}
\mathbb{E} \left[ \left( Y - \hat{\mu} \right)^2 \right] = \sigma^2 + \mathbb{V} \left[ \hat{\mu} \right] + \left(\mu - \mathbb{E} \left[ \hat{\mu} \right] \right)^2.
\end{equation*}
\label{ex:MSE_decomp}
\end{example}

\subsection{Classification Log-Likelihood as a Special Special Case}
\label{sec:cll}

Even though classification is a particularly relevant task for Deep Learning, the current literature states the log-likelihood decomposition via the log-probabilities \citep{brofos2019bias, brinda2019holder, yang2020rethinking}.
Since neural networks output logits, a normalization step is required.
This step is cumbersome to compute and hinders theoretical analysis.
In the following, we will construct the Bregman Information for classification via Proposition \ref{prop:exp_fam_decomp} similar to Example \ref{ex:MSE_decomp}.
Surprisingly, the Bregman Information measures the variability of the prediction in the logit space and does not require normalization to the log-probability space.

\begin{corollary}
For logit prediction $\hat{z} \in \mathbb{R}^k$ and target $Y \sim Q$ with $k$ classes, we have
\begin{equation*}
    \underbrace{\mathbb{E} \left[ - \ln \mathrm{sm}_Y \! \left( \hat{z} \right) \right]}_{\text{Classif. NLL}} \! = \!
    \underbrace{H \! \left( Q \right)}_{\text{Noise}} +
    \underbrace{ \mathbb{B}_{\mathrm{LSE}} \left[ \hat{z} \right]}_{\text{"Variance"}} + \underbrace{ d_{\mathrm{LSE}} \left( \mathrm{sm}^{-1} \! \left( Q \right), \mathbb{E} \left[ \hat{z} \right] \right)}_{\text{Bias}}
\end{equation*}
with the LogSumExp function $\mathrm{LSE} \left(x_1, \dots, x_n \right) = \ln \sum_{i=1}^n e^{x_i}$, the softmax function $\mathrm{sm} = \nabla \mathrm{LSE}$, and Shannon entropy $H$.
\label{cor:classifnll}
\end{corollary}

The proof in Appendix \ref{app:proofs} combines Proposition \ref{prop:exp_fam_decomp} with the properties for categorical distributions presented in Table \ref{tab:expfam}.
Note that $\mathrm{sm}^{-1}$ maps only into a $k-1$ dimensional subspace of $\mathbb{R}^k$. \\
The Bregman Information acting in the logit space is a convenient surprise for Deep Learning applications.
Almost all neural networks used for classification do not output probabilities but logits.
Consequently, we do not require normalizing the neural network outputs to compute the Bregman Information. \\
Further, we can express the Bregman Information in the binary classification case through the softplus function since $\mathbb{B}_{\mathrm{LSE}} \left[ \left( \hat{z}_1, \hat{z}_2 \right)^\intercal \right] = \mathbb{B}_{\sigma_+} \left[ \hat{z}_2 - \hat{z}_1 \right]$.
The Bregman Information generated by the softplus function is illustrated in Figure \ref{fig:B_LSE}.
We chose the depicted random variable such that the Jensen gap visualizes geometrically.

\subsection{Ensemble Predictions}
\label{sec:ens}

Using ensembles as predictions in Machine Learning is a central aspect of many successful architectures, like Random Forest \citep{breiman2001random}, XGBoost \citep{chen2016xgboost}, Deep Ensembles (DE) \citep{lakshminarayanan2017simple}, or Test-Time augmentation (TTA) \citep{wang2019aleatoric}.
DE and TTA show strong empirical results even though they do not sample the training data, unlike Random Forests or XGBoost.
\cite{NEURIPS2020_7d420e2b} use the law of total variance to study the effect of different noise sources and \cite{gupta2022ensembles} generalize this law to vector-based Bregman divergences.
In this section, we show that an equivalent law also holds for (functional) Bregman Information and use it to compare finite sized ensembles.
To do so, we require a definition for conditional Bregman Information based on \cite{banerjee2004optimal}, which also includes the functional case.
%We apply the theoretical statements t argue how an ensemble reduces a specific source of uncertainty via a general law of total variance.
%\cite{banerjee2004optimal} introduce the term conditional Bregman Information, which we define to also include the functional case.

\begin{definition}
    Let $\phi \colon U \to \mathbb{R}$ be a convex function.
    We define the (functional) \textbf{conditional Bregman Information} (generated by $\phi$) of a random variable $X$ with realizations in $U$ given another random variable $Y$ as
    \begin{equation*}
        \mathbb{B}_\phi \left[ X \mid Y \right] = \mathbb{E} \left[ \phi \left( X \right) \mid Y \right] - \phi \left( \mathbb{E} \left[ X \mid Y \right] \right).
    \end{equation*}
\label{def:cBI}
\end{definition}

Similar to Bregman Information, it holds that $\mathbb{B}_\phi \left[ X \mid Y \right] = \mathbb{E} \left[ d_{\phi} \left( \mathbb{E} \left[ X \mid Y \right], X \right) \mid Y \right]$ for differentiable $\phi$ \citep{banerjee2004optimal}.
%Again, the subgradient does not influence the functional case.
The conditional variance appears by setting $\phi \left( x \right) = x^2$.
We can now contribute the following.

\begin{proposition}[Properties of Bregman Information]
The general law of total variance for a (functional) Bregman Information $\mathbb{B}_G$ and random variables $X$ and $Y$ is given by
\begin{equation}
    \mathbb{B}_G \left[ X \right] = \mathbb{E} \left[ \mathbb{B}_G \left[ X \mid Y \right] \right] + \mathbb{B}_G \left[ \mathbb{E} \left[ X \mid Y \right] \right].
\label{eq:tv}
\end{equation}
For i.i.d. random variables $X_1, \dots, X_{2^n}$ and (strictly) convex $G$, we have 
\begin{equation}
    \mathbb{B}_G \left[ \frac{1}{2^n} \sum_{i=1}^{2^n} X_i \right] \overset{(<)}{\leq} \mathbb{B}_G \left[ \frac{1}{2^{n-1}} \sum_{i=1}^{2^{n-1}} X_i \right],
\label{eq:nv_ens}
\end{equation}

and if $G$ is also continuous with $n \to \infty$ then %$\mathbb{B}_G \left[ \frac{1}{n} \sum_{i=1}^n X_i \right] \overset{a.s.}{\longrightarrow} 0$.
\begin{equation}
    \mathbb{B}_G \left[ \frac{1}{n} \sum_{i=1}^n X_i \right] \overset{a.s.}{\longrightarrow} 0.
\label{eq:limes}
\end{equation} 
\label{prop:BI_props}
\end{proposition}

The last two properties also hold in the case of conditional Bregman Information. \\
We now apply Corollary \ref{prop:BI_props} in the context of ensemble predictions.
To simplify the argument, we assume the context of an exponential family. But, the statements also hold for non-parametric scenarios.
Let a prediction $\theta_{W, D}$ in the natural parameter space depend on $W$ and $D$ as two sources of variability.
In the case of Deep Ensembles, $D$ is the training data, and $W$ is the weight initialization.
For TTA, $W$ is the input variation, like angle or rotation.
If we compute an ensemble prediction $\hat{\theta}_{D}^{(n)} = 2^{-n} \sum_{i=1}^{2^n} \theta_{W_i,D}$
by sampling $W$ while keeping $D$ fixed, we marginalize out the uncertainty in the prediction stemming from $W$, since if $n \to \infty$, then
\begin{equation}
    \mathbb{B}_{A} \left[ \hat{\theta}_{D}^{(n)} \right] = \mathbb{B}_{A} \left[ \mathbb{E} \left[ \theta_{W,D} \mid D \right] \right] + \underbrace{\mathbb{E} \left[ \mathbb{B}_{A} \left[ \hat{\theta}_{D}^{(n)} \mid D \right] \right]}_{\longrightarrow 0}.%\overset{n \to \infty}{\longrightarrow} 0}.
\end{equation}

As $n$ does not influence the bias term, the expected negative log-likelihood reduces by the amount of conditional Bregman Information of $W$.
Together with Equation \eqref{eq:nv_ens}, this results in

\begin{equation}
    \mathbb{E} \left[ - \ln p_{\hat{\theta}_{D}^{(n)}} \left( Y \right) \right] < \mathbb{E} \left[ - \ln p_{\hat{\theta}_{D}^{(n-1)}} \left( Y \right) \right].
\label{eq:ens_improv}
\end{equation}
For $n=0$ we recover the case of a single prediction.
It follows that an ensemble always has better expected performance than a single model as long as the ensemble members are generated by a true source of uncertainty.

\begin{figure}
\vspace{.2in}
    \centering
    \resizebox{0.9\columnwidth}{!}{%
    \begin{tikzpicture}
        \datavisualization [school book axes,
        visualize as smooth line/.list={one},
        %all axes={ticks={tick typesetter/.code=}},
        y axis={label={}, ticks={step=2}},
        x axis={label={$x$}, ticks=few},
        one={style={red, thick}}]
        data [set=one,format=function] {
            var x : interval [-3:4];
            func y = ln(1 + exp(\value x) );
        };
        %[font=\small]
        \node at (2.5, 3.75) {$\ln \left( 1 + e^x \right)$};
        \draw [thick] (-2, 0.1269) -- (3, 3.0486);
        \draw [thick] (-2, -0.4868) -- (3, 2.4349);
        
        \draw [dashed, thick] (3, -0.6) -- (3, 3.0486);
        \draw [dashed, thick] (-2, -0.6) -- (-2, 0.1269);
        \draw [decorate, decoration = {brace,mirror}] (-2, -0.7) --  (3, -0.7);
        \node at (0.5, -1) {$\left(1 - \alpha \right) \cdot 100\%$ CI};

        \draw [dashed, thick] (0.5, 0.9741) -- (0.5, -0.2);
        \node[inner sep=1pt,] at (0.5, -0.4) {$\mathbb{E} \left[ \hat{z} \right]$};
        \node at (3.95, 2.72) {$\Big\} = \frac{1}{\alpha} \mathbb{B}_{\sigma_+} \left[ \hat{z} \right]$};
    \end{tikzpicture}
    }
\vspace{.2in}
\caption{
Illustration of the confidence interval for binary classification.
The CI covers the area where the shifted tangent at $\mathbb{E} \left[ \hat{z} \right]$ is larger than the generating convex function.
This illustration generalizes to higher dimensions, where the CI is not an interval anymore in the strict sense but a region.
}
\label{fig:CI_BI}
\end{figure}

\subsection{Confidence Regions Based on Bregman Information}
\label{sec:CI}

In practical applications, one might be interested in a confidence interval for a prediction.
For example, we could ask for an interval that covers a given prediction in 95\% of the cases concerning the randomness in the training data.
If we have a high-dimensional or functional prediction, we are not using an interval anymore but a convex set.
Consequently, we refer to it as a confidence region.
We can construct such a region as in the following.
Applying Markov's inequality to the Bregman divergence $d_G$ between a mean and its random variable $\mathcal{X}$ gives $\mathbb{P} \left( d_G \left( \mathbb{E} \left[ X \right], X \right) \geq c \right) \leq \frac{1}{c} \mathbb{B}_G \left[ X \right]$. \\ 
Setting $c = \frac{1}{\alpha} \mathbb{B}_G \left[ X \right]$ results in $d_G \left( \mathbb{E} \left[ X \right], X \right) \leq \frac{1}{\alpha} \mathbb{B}_G \left[ X \right]$ having at least probability $\left(1 - \alpha \right)$.
Consequently, we are given a $\left(1 - \alpha \right) \cdot 100\%$-confidence region by
\begin{equation}
    \mathrm{CR}_{\left(1 - \alpha \right)} = \left\{ x \in \mathcal{X} \mid d_G \left( \mathbb{E} \left[ X \right], x \right) \leq \frac{ \mathbb{B}_G \left[ X \right]}{\alpha} \right\}
\end{equation}
with support $\mathcal{X}$ of $X$.
An illustration for the binary classification case is depicted in Figure \ref{fig:CI_BI}, where we construct a confidence interval in the one-dimensional logit space. \\
Further, we demonstrate the confidence regions of a neural network trained on the Iris dataset.
We estimate the Bregman Information by training models with different weight initializations.
The resulting regions in Figure \ref{fig:iris} illustrate the influence of the weight initialization on the variability of the prediction.

\begin{figure}
\vspace{.2in}
\centerline{\includegraphics[width=0.8\linewidth]{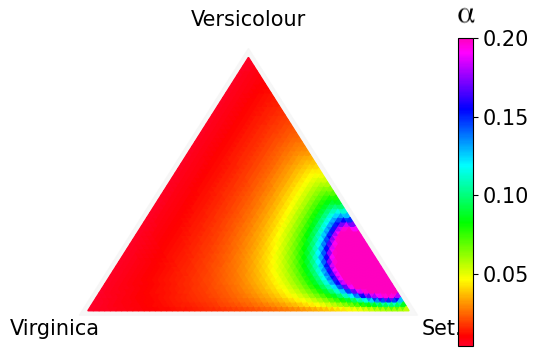}}
\vspace{.25in}
\caption{
Confidence regions of a prediction transformed into the simplex for various alphas.
The model is a simple neural network fitted on the Iris dataset.
}
\label{fig:iris}
\end{figure}

\begin{figure*}
\vspace{.1in}
\centerline{\includegraphics[width=\linewidth]{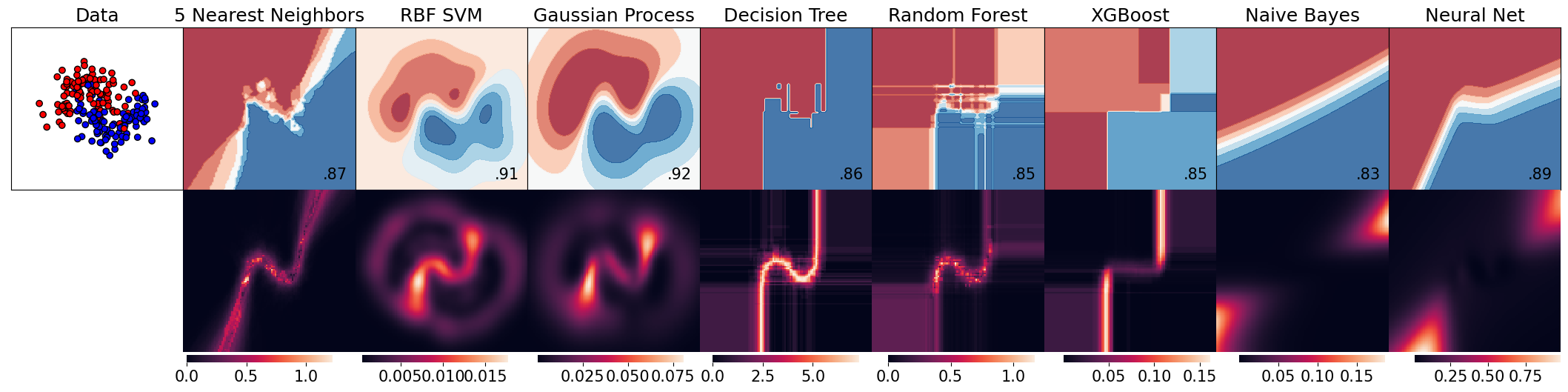}}
\vspace{.2in}
\caption{
Several classifiers are trained on a simulated toy task.
\textbf{Top}: Classifier predictions for a frame of the input space. The accuracy is displayed in each bottom right corner.
\textbf{Bottom}: Bregman Information of these classifiers in the identical space based on multiple training runs.
}
\label{fig:BI_toy_0}
\end{figure*}

\begin{figure}
\vspace{.1in}
\centerline{\includegraphics[width=\columnwidth]{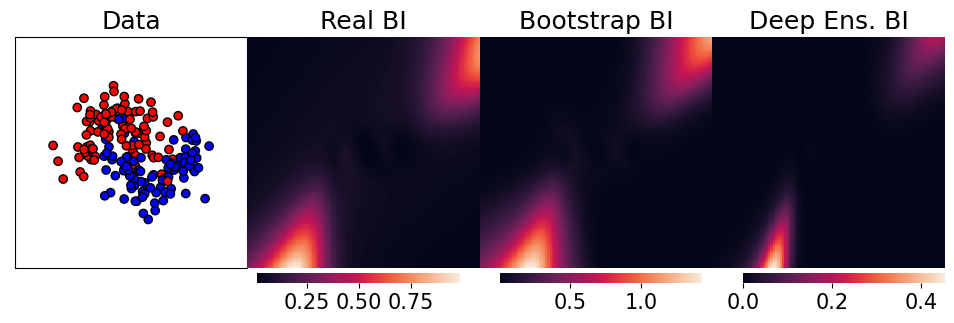}}
\vspace{.2in}
\caption{
    Comparing Bregman Information estimated via bootstrapping and Deep Ensembles for a neural network and a single training set with the ground truth ('Real BI').
}
\label{fig:BI_approx}
\end{figure}

\section{EXPERIMENTS}
\label{sec:exp}

In this section, we evaluate the Bregman Information and its approximations for classifiers via various experiments.
Throughout all evaluations, we use the estimator $\hat{\mathbb{B}}^{(n)}_{\mathrm{LSE}} = \frac{1}{n} \sum_{i=1}^n \mathrm{LSE} \left( \hat{z}_i \right) - \mathrm{LSE} \left( \frac{1}{n} \sum_{i=1}^n \hat{z}_i \right)$.
First, we evaluate typical classifiers on toy tasks, where we can simulate the ground truth.
Accessing the data generation process also allows us to compare different approximation schemes of the model Bregman Information.
Based on the insights, we provide experiments on corrupted CIFAR-10 and ImageNet and investigate how we may use the estimated Bregman Information for better predictive performance under domain drift.

\subsection{Bregman Information of Common Classifiers}

We assess the Bregman Information of the classifiers k-nearest neighbors, Support Vector Machine, Gaussian Process, Decision Tree, Random Forest, XGBoost, Naive Bayes, and neural network \citep{pml1Book}.
We create toy tasks and sample an arbitrary number of training datasets.
For each sampled dataset, we fit all of the previous classifiers.
This way, we can approximate the ground truth Bregman Information of each classifier arbitrarily well for growing sample size.
Empirically, we consider 64 samples as a sufficiently close approximation.
Further details appear in Appendix \ref{app:exp}.
We plot the Bregman Information in a close region around the data distribution in the second row of Figure \ref{fig:BI_toy_0}.
The first row depicts a single exemplary sample of a training set and the confidence score of each respective classifier to put the Bregman Information plots in a meaningful perspective.
As we can see, most models show a high uncertainty along the decision boundary.
For example, the uncertainty of SVMs and Gaussian Processes is restricted to an area close to the training distribution. 
The Bregman Information of KNN and decision-tree-based models suggests a narrow decision boundary of high uncertainty even where no data is present. 
In contrast, the neural network and the naive Bayes classifier show increasing uncertainty towards out-of-domain regions along the decision boundary.\\
Figure \ref{fig:BI_approx} illustrates that Deep Ensembles and Bootstrapping \citep{efron1994introduction} can serve as practical approximations for the ground truth Bregman Information.
More simulations of toy tasks and MC Dropout \citep{gal2016dropout} for approximation are presented in Appendix \ref{app:exp}.

\begin{figure*}[t]
\vspace{.1in}
\centering
\includegraphics[width=\textwidth]{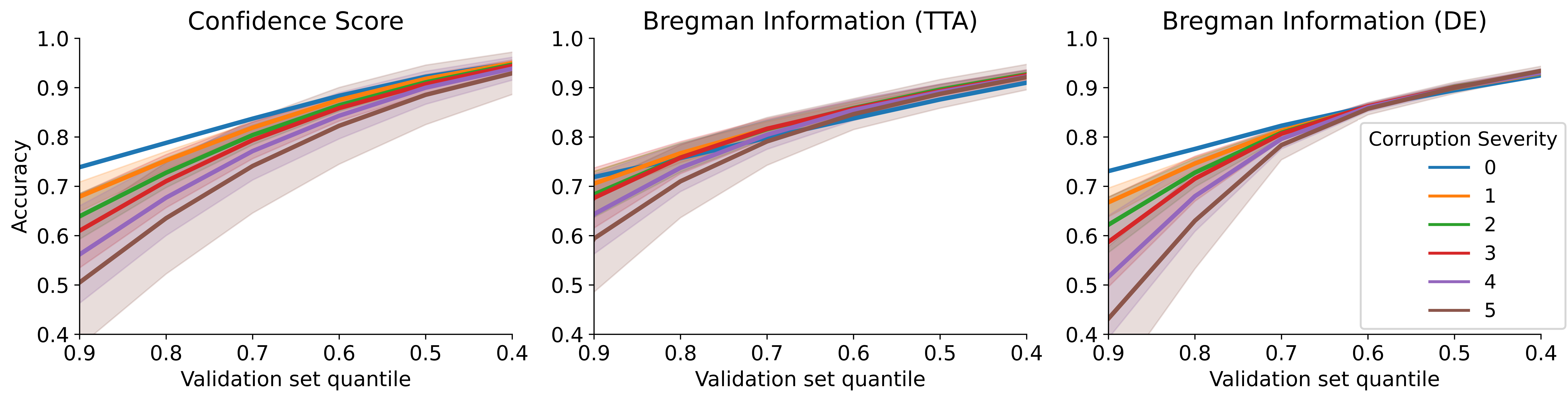}
%\centerline{\fbox{This figure intentionally left non-blank}}
\vspace{.1in}
\caption{
    Accuracy after discarding for different types of uncertainty on ImageNet and for varying levels of corruption on ImageNet-C.
    The error bounds stem from the corruption type.
    The Bregman Information is more robust to corruption severity and corruption type for stricter thresholds compared to Confidence scores.
    \textbf{Left:} Using predicted confidence as uncertainty estimate.
    \textbf{Middle:} Using Bregman Information based on test-time augmentation for uncertainty.
    \textbf{Right:} Again, using Bregman Information but based on Deep Ensembles.
}
\label{fig:ood_imgnet}
\end{figure*}

Based on our observations, the Bregman Information approximated by an ensemble could be a good proxy of the neural network's uncertainty even in regions we have not seen in the training data.
We could differ between in-domain and out-of-domain instances, especially if the decision boundary in a high-dimensional space, like images, is 'open' towards multiple directions.

\subsection{Out-of-Distribution Detection via Bregman Information}

In this section, we use the Bregman Information of ensemble predictions for out-of-domain detection on image data.
We propose a procedure along the following steps (c.f. Algorithm \ref{alg:BI}).
First, we require an ensemble of predictions to approximate the unknown uncertainty.
Next, we compute the Bregman Information for each data instance in the validation set according to the ensemble predictions.
We now have a set of instance-level Bregman Information for the in-domain data.
Then, we compute the empirical quantiles of the Bregman Information values in the validation set.
The central assumption is that out-of-domain data results in generally high Bregman Information.
Consequently, thresholding our classification based on a chosen quantile should discard out-of-domain instances while only discarding a controlled amount of in-domain ones.
For example, if we pick the 0.9-quantile, then 90\% of the validation data is considered in-domain.
In the ideal case, we only classify out-of-domain data instances close to the in-domain ones with a similar accuracy. \\
We assess the proposed procedure on CIFAR-10, ImageNet, and their corrupted versions (CIFAR-10-C and ImageNet-C) from \cite{hendrycks2019robustness}.
These corruptions are provided in different types and levels of severity, ranging from one to five, where five is the worst corrupted (c.f. Appendix \ref{app:exp}).
We evaluate Deep Ensembles and TTA as ensemble approaches.
For ImageNet, we used pre-trained ResNet-50 models from \cite{ashukha2020pitfalls}.
To not skew the results through a performance gap between single models and model ensembles, we only use a single ResNet as classifier and use the ensemble only for estimating the Bregman Information. \\
As a baseline, we use Algorithm \ref{alg:BI} with the predicted confidence score for uncertainty estimation (c.f. Algorithm \ref{alg:Conf} in Appendix \ref{app:exp}).
The results are depicted in Figures \ref{fig:ood_cif10} and \ref{fig:ood_imgnet}.
As we can see in the ImageNet case, upwards from the 0.6-quantile, our procedure successfully detects out-of-domain instances, which would have significantly lower accuracy than in-domain data.
Consequently, the classified out-of-domain instances do not degrade the model performance.
Further, the Bregman Information is very robust to the severity and the type of corruption. 
Contrary, the confidence score gives increasingly worse performance estimates with rising severity.

\begin{algorithm}
\caption{Classifying with uncertainty threshold}
\label{alg:BI}
\begin{algorithmic}
\Require Validation set $\mathcal{D}$, model, ensemble, $q \in \left[0, 1 \right]$, test instance $x^\prime$
%\Ensure $y = x^n$
\State BIs $\gets$ [BI(ensemble($x$)) for $x \in \mathcal{D}$]
\State threshold $ \gets $ quantile(BIs, q)
\If{BI(ensemble($x^\prime$)) $>$ threshold}
    \State label as OOD \Comment{Warning in real-world application}
\Else
    \State return model($x^\prime$)
\EndIf
\end{algorithmic}
\end{algorithm}

\subsection{Practical Limitations and Future Work}

The main contribution in this work is of theoretical nature.
Even though we demonstrate promising empirical results, more research is required for practical applications.
We performed the experiments to suggest in what research areas the Bregman Information can potentially improve current methodologies, especially in the OOD setting.
We leave a comparison in the context of state-of-the-art methods for future research.
The different approximations of the Bregman Information are also preliminary.
More extensive and thorough benchmarks may show better alternative approaches.
Further, the used estimator $\hat{\mathbb{B}}^{(n)}_{\mathrm{LSE}}$ is only asymptotically unbiased, and underestimates the theoretical quantity.
Corollary \ref{cor:classifnll} and Section \ref{sec:ens} may suggest to average ensembles in the logit space.
But, as \cite{gupta2022ensembles} demonstrated, averaging in the probability space may impact the bias in a positive way and reduce the overall error further than averaging in the logit space.

\section{CONCLUSION}

Through properties of functional Bregman divergences, we introduced a general bias-variance decomposition for strictly proper scores.
We discovered that the Bregman Information always represents the variance term.
Our decomposition generalizes and provides new formulations of the exponential family and classification log-likelihood decomposition.
Specifically, we formulated the classification case for logit predictions without requiring normalization to the log space.
Further, we derived new general insights for ensemble predictions and how we construct confidence regions for predictions.
As a real-world application, we proposed Bregman Information as uncertainty measure to facilitate model performance under domain drift via out-of-distribution detection for all degrees of severity and types of corruption.

\medskip

{
\small

\bibliography{main}

\begin{thebibliography}{52}
\providecommand{\natexlab}[1]{#1}
\providecommand{\url}[1]{\texttt{#1}}
\expandafter\ifx\csname urlstyle\endcsname\relax
  \providecommand{\doi}[1]{doi: #1}\else
  \providecommand{\doi}{doi: \begingroup \urlstyle{rm}\Url}\fi

\bibitem[Adlam \& Pennington(2020)Adlam and Pennington]{NEURIPS2020_7d420e2b}
Adlam, B. and Pennington, J.
\newblock Understanding double descent requires a fine-grained bias-variance
  decomposition.
\newblock In Larochelle, H., Ranzato, M., Hadsell, R., Balcan, M.~F., and Lin,
  H. (eds.), \emph{Advances in Neural Information Processing Systems},
  volume~33, pp.\  11022--11032. Curran Associates, Inc., 2020.
\newblock URL
  \url{https://proceedings.neurips.cc/paper/2020/file/7d420e2b2939762031eed0447a9be19f-Paper.pdf}.

\bibitem[Ashukha et~al.(2020)Ashukha, Lyzhov, Molchanov, and
  Vetrov]{ashukha2020pitfalls}
Ashukha, A., Lyzhov, A., Molchanov, D., and Vetrov, D.
\newblock Pitfalls of in-domain uncertainty estimation and ensembling in deep
  learning.
\newblock In \emph{International Conference on Learning Representations}, 2020.

\bibitem[Banerjee et~al.(2004)Banerjee, Guo, and Wang]{banerjee2004optimal}
Banerjee, A., Guo, X., and Wang, H.
\newblock Optimal bregman prediction and jensen's equality.
\newblock In \emph{International Symposium onInformation Theory, 2004. ISIT
  2004. Proceedings.}, pp.\  169. IEEE, 2004.

\bibitem[Banerjee et~al.(2005)Banerjee, Merugu, Dhillon, Ghosh, and
  Lafferty]{banerjee2005clustering}
Banerjee, A., Merugu, S., Dhillon, I.~S., Ghosh, J., and Lafferty, J.
\newblock Clustering with bregman divergences.
\newblock \emph{Journal of machine learning research}, 6\penalty0 (10), 2005.

\bibitem[Bregman(1967)]{bregman1967relaxation}
Bregman, L.
\newblock The relaxation method of finding the common point of convex sets and
  its application to the solution of problems in convex programming.
\newblock \emph{USSR Computational Mathematics and Mathematical Physics},
  7\penalty0 (3):\penalty0 200 -- 217, 1967.
\newblock ISSN 0041-5553.
\newblock \doi{https://doi.org/10.1016/0041-5553(67)90040-7}.
\newblock URL
  \url{http://www.sciencedirect.com/science/article/pii/0041555367900407}.

\bibitem[Breiman(2001)]{breiman2001random}
Breiman, L.
\newblock Random forests.
\newblock \emph{Machine learning}, 45\penalty0 (1):\penalty0 5--32, 2001.

\bibitem[Brinda et~al.(2019)Brinda, Klusowski, and Yang]{brinda2019holder}
Brinda, W., Klusowski, J.~M., and Yang, D.
\newblock H{\"o}lder’s identity.
\newblock \emph{Statistics \& Probability Letters}, 148:\penalty0 150--154,
  2019.

\bibitem[Brofos et~al.(2019)Brofos, Shu, and Lederman]{brofos2019bias}
Brofos, J., Shu, R., and Lederman, R.~R.
\newblock A bias-variance decomposition for bayesian deep learning.
\newblock In \emph{NeurIPS 2019 Workshop on Bayesian Deep Learning}, 2019.

\bibitem[Capi{\'n}ski \& Kopp(2004)Capi{\'n}ski and Kopp]{capinski2004measure}
Capi{\'n}ski, M. and Kopp, P.~E.
\newblock \emph{Measure, integral and probability}, volume~14.
\newblock Springer, 2004.

\bibitem[Chen \& Guestrin(2016)Chen and Guestrin]{chen2016xgboost}
Chen, T. and Guestrin, C.
\newblock Xgboost: A scalable tree boosting system.
\newblock In \emph{Proceedings of the 22nd acm sigkdd international conference
  on knowledge discovery and data mining}, pp.\  785--794, 2016.

\bibitem[Csiszar(1991)]{csiszar1991least}
Csiszar, I.
\newblock Why least squares and maximum entropy? an axiomatic approach to
  inference for linear inverse problems.
\newblock \emph{The annals of statistics}, 19\penalty0 (4):\penalty0
  2032--2066, 1991.

\bibitem[Dawid(2007)]{dawid2007geometry}
Dawid, A.~P.
\newblock The geometry of proper scoring rules.
\newblock \emph{Annals of the Institute of Statistical Mathematics},
  59\penalty0 (1):\penalty0 77--93, 2007.

\bibitem[Della~Pietra et~al.(2002)Della~Pietra, Della~Pietra, and
  Lafferty]{della2002duality}
Della~Pietra, S., Della~Pietra, V., and Lafferty, J.
\newblock Duality and auxiliary functions for bregman distances (revised).
\newblock Technical report, CARNEGIE-MELLON UNIV PITTSBURGH PA SCHOOL OF
  COMPUTER SCIENCE, 2002.

\bibitem[Domingos(2000)]{domingos2000unified}
Domingos, P.
\newblock A unified bias-variance decomposition.
\newblock In \emph{Proceedings of 17th international conference on machine
  learning}, pp.\  231--238. Morgan Kaufmann Stanford, 2000.

\bibitem[Efron(1994)]{efron1994introduction}
Efron, B.
\newblock \emph{An introduction to the bootstrap}.
\newblock CRC press, 1994.

\bibitem[Fenchel(1949)]{fenchel1949conjugate}
Fenchel, W.
\newblock On conjugate convex functions.
\newblock \emph{Canadian Journal of Mathematics}, pp.\ ~73, 1949.

\bibitem[Frigyik et~al.(2006)Frigyik, Srivastava, and
  Gupta]{https://doi.org/10.48550/arxiv.cs/0611123}
Frigyik, B.~A., Srivastava, S., and Gupta, M.~R.
\newblock Functional bregman divergence and bayesian estimation of
  distributions, 2006.
\newblock URL \url{https://arxiv.org/abs/cs/0611123}.

\bibitem[Gal \& Ghahramani(2016)Gal and Ghahramani]{gal2016dropout}
Gal, Y. and Ghahramani, Z.
\newblock Dropout as a bayesian approximation: Representing model uncertainty
  in deep learning.
\newblock In \emph{international conference on machine learning}, pp.\
  1050--1059. PMLR, 2016.

\bibitem[Glorot et~al.(2011)Glorot, Bordes, and Bengio]{glorot2011deep}
Glorot, X., Bordes, A., and Bengio, Y.
\newblock Deep sparse rectifier neural networks.
\newblock In \emph{Proceedings of the fourteenth international conference on
  artificial intelligence and statistics}, pp.\  315--323. JMLR Workshop and
  Conference Proceedings, 2011.

\bibitem[Gneiting \& Raftery(2007)Gneiting and Raftery]{gneitingscores}
Gneiting, T. and Raftery, A.~E.
\newblock Strictly proper scoring rules, prediction, and estimation.
\newblock \emph{Journal of the American Statistical Association}, 102\penalty0
  (477):\penalty0 359--378, 2007.
\newblock \doi{10.1198/016214506000001437}.
\newblock URL \url{https://doi.org/10.1198/016214506000001437}.

\bibitem[Gr{\"u}nwald \& Dawid(2004)Gr{\"u}nwald and Dawid]{grunwald2004game}
Gr{\"u}nwald, P.~D. and Dawid, A.~P.
\newblock Game theory, maximum entropy, minimum discrepancy and robust bayesian
  decision theory.
\newblock \emph{the Annals of Statistics}, 32\penalty0 (4):\penalty0
  1367--1433, 2004.

\bibitem[Guo et~al.(2017)Guo, Pleiss, Sun, and Weinberger]{guo2017calibration}
Guo, C., Pleiss, G., Sun, Y., and Weinberger, K.~Q.
\newblock On calibration of modern neural networks.
\newblock In \emph{International Conference on Machine Learning}, pp.\
  1321--1330. PMLR, 2017.

\bibitem[Gupta et~al.(2022)Gupta, Smith, Adlam, and Mariet]{gupta2022ensembles}
Gupta, N., Smith, J., Adlam, B., and Mariet, Z.~E.
\newblock Ensembles of classifiers: a bias-variance perspective.
\newblock \emph{Transactions on Machine Learning Research}, 2022.
\newblock ISSN 2835-8856.
\newblock URL \url{https://openreview.net/forum?id=lIOQFVncY9}.

\bibitem[Haggenmüller et~al.(2021)Haggenmüller, Maron, Hekler, Utikal,
  Barata, Barnhill, Beltraminelli, Berking, Betz-Stablein, Blum, Braun, Carr,
  Combalia, Fernandez-Figueras, Ferrara, Fraitag, French, Gellrich, Ghoreschi,
  Goebeler, Guitera, Haenssle, Haferkamp, Heinzerling, Heppt, Hilke,
  Hobelsberger, Krahl, Kutzner, Lallas, Liopyris, Llamas-Velasco, Malvehy,
  Meier, Müller, Navarini, Navarrete-Dechent, Perasole, Poch, Podlipnik,
  Requena, Rotemberg, Saggini, Sangueza, Santonja, Schadendorf, Schilling,
  Schlaak, Schlager, Sergon, Sondermann, Soyer, Starz, Stolz, Vale, Weyers,
  Zink, Krieghoff-Henning, Kather, {von Kalle}, Lipka, Fröhling, Hauschild,
  Kittler, and Brinker]{HAGGENMULLER2021202}
Haggenmüller, S., Maron, R.~C., Hekler, A., Utikal, J.~S., Barata, C.,
  Barnhill, R.~L., Beltraminelli, H., Berking, C., Betz-Stablein, B., Blum, A.,
  Braun, S.~A., Carr, R., Combalia, M., Fernandez-Figueras, M.-T., Ferrara, G.,
  Fraitag, S., French, L.~E., Gellrich, F.~F., Ghoreschi, K., Goebeler, M.,
  Guitera, P., Haenssle, H.~A., Haferkamp, S., Heinzerling, L., Heppt, M.~V.,
  Hilke, F.~J., Hobelsberger, S., Krahl, D., Kutzner, H., Lallas, A., Liopyris,
  K., Llamas-Velasco, M., Malvehy, J., Meier, F., Müller, C.~S., Navarini,
  A.~A., Navarrete-Dechent, C., Perasole, A., Poch, G., Podlipnik, S., Requena,
  L., Rotemberg, V.~M., Saggini, A., Sangueza, O.~P., Santonja, C.,
  Schadendorf, D., Schilling, B., Schlaak, M., Schlager, J.~G., Sergon, M.,
  Sondermann, W., Soyer, H.~P., Starz, H., Stolz, W., Vale, E., Weyers, W.,
  Zink, A., Krieghoff-Henning, E., Kather, J.~N., {von Kalle}, C., Lipka,
  D.~B., Fröhling, S., Hauschild, A., Kittler, H., and Brinker, T.~J.
\newblock Skin cancer classification via convolutional neural networks:
  systematic review of studies involving human experts.
\newblock \emph{European Journal of Cancer}, 156:\penalty0 202--216, 2021.
\newblock ISSN 0959-8049.
\newblock \doi{https://doi.org/10.1016/j.ejca.2021.06.049}.
\newblock URL
  \url{https://www.sciencedirect.com/science/article/pii/S0959804921004445}.

\bibitem[Hansen \& Heskes(2000)Hansen and Heskes]{hansen2000general}
Hansen, J.~V. and Heskes, T.
\newblock General bias/variance decomposition with target independent variance
  of error functions derived from the exponential family of distributions.
\newblock In \emph{Proceedings 15th International Conference on Pattern
  Recognition. ICPR-2000}, volume~2, pp.\  207--210. IEEE, 2000.

\bibitem[He et~al.(2016)He, Zhang, Ren, and Sun]{he2016deep}
He, K., Zhang, X., Ren, S., and Sun, J.
\newblock Deep residual learning for image recognition.
\newblock In \emph{Proceedings of the IEEE conference on computer vision and
  pattern recognition}, pp.\  770--778, 2016.

\bibitem[Hendrickson \& Buehler(1971)Hendrickson and
  Buehler]{hendrickson1971proper}
Hendrickson, A.~D. and Buehler, R.~J.
\newblock Proper scores for probability forecasters.
\newblock \emph{The Annals of Mathematical Statistics}, 42\penalty0
  (6):\penalty0 1916--1921, 1971.

\bibitem[Hendrycks \& Dietterich(2019)Hendrycks and
  Dietterich]{hendrycks2019robustness}
Hendrycks, D. and Dietterich, T.
\newblock Benchmarking neural network robustness to common corruptions and
  perturbations.
\newblock \emph{Proceedings of the International Conference on Learning
  Representations}, 2019.

\bibitem[Heskes(1998)]{heskes1998bias}
Heskes, T.
\newblock Bias/variance decompositions for likelihood-based estimators.
\newblock \emph{Neural Computation}, 10\penalty0 (6):\penalty0 1425--1433,
  1998.

\bibitem[James \& Hastie(1997)James and Hastie]{james1997generalizations}
James, G. and Hastie, T.
\newblock Generalizations of the bias/variance decomposition for prediction
  error.
\newblock \emph{Dept. Statistics, Stanford Univ., Stanford, CA, Tech. Rep},
  1997.

\bibitem[James(2003)]{james2003variance}
James, G.~M.
\newblock Variance and bias for general loss functions.
\newblock \emph{Machine learning}, 51\penalty0 (2):\penalty0 115--135, 2003.

\bibitem[Jones \& Byrne(1990)Jones and Byrne]{jones1990general}
Jones, L.~K. and Byrne, C.~L.
\newblock General entropy criteria for inverse problems, with applications to
  data compression, pattern classification, and cluster analysis.
\newblock \emph{IEEE transactions on Information Theory}, 36\penalty0
  (1):\penalty0 23--30, 1990.

\bibitem[Katsaouni et~al.(2021)Katsaouni, Tashkandi, Wiese, and
  Schulz]{KatsaouniTashkandiWieseSchulz+2021+871+885}
Katsaouni, N., Tashkandi, A., Wiese, L., and Schulz, M.~H.
\newblock Machine learning based disease prediction from genotype data.
\newblock \emph{Biological Chemistry}, 402\penalty0 (8):\penalty0 871--885,
  2021.
\newblock \doi{doi:10.1515/hsz-2021-0109}.
\newblock URL \url{https://doi.org/10.1515/hsz-2021-0109}.

\bibitem[Krizhevsky(2009)]{krizhevsky2009learning}
Krizhevsky, A.
\newblock Learning multiple layers of features from tiny images.
\newblock Master's thesis, University of Toronto, 2009.

\bibitem[Kurdila \& Zabarankin(2006)Kurdila and Zabarankin]{kurdila2006convex}
Kurdila, A.~J. and Zabarankin, M.
\newblock \emph{Convex functional analysis}.
\newblock Springer Science \& Business Media, 2006.

\bibitem[Lakshminarayanan et~al.(2017)Lakshminarayanan, Pritzel, and
  Blundell]{lakshminarayanan2017simple}
Lakshminarayanan, B., Pritzel, A., and Blundell, C.
\newblock Simple and scalable predictive uncertainty estimation using deep
  ensembles.
\newblock \emph{Advances in neural information processing systems}, 30, 2017.

\bibitem[Murphy(2022)]{pml1Book}
Murphy, K.~P.
\newblock \emph{Probabilistic Machine Learning: An introduction}.
\newblock MIT Press, 2022.
\newblock URL \url{probml.ai}.

\bibitem[Ovadia et~al.(2019)Ovadia, Fertig, Ren, Nado, Sculley, Nowozin,
  Dillon, Lakshminarayanan, and Snoek]{ovadia2019can}
Ovadia, Y., Fertig, E., Ren, J., Nado, Z., Sculley, D., Nowozin, S., Dillon,
  J., Lakshminarayanan, B., and Snoek, J.
\newblock Can you trust your model's uncertainty? evaluating predictive
  uncertainty under dataset shift.
\newblock \emph{Advances in neural information processing systems}, 32, 2019.

\bibitem[Ovcharov(2015)]{ovcharov2015existence}
Ovcharov, E.~Y.
\newblock Existence and uniqueness of proper scoring rules.
\newblock \emph{J. Mach. Learn. Res.}, 16:\penalty0 2207--2230, 2015.

\bibitem[Ovcharov(2018)]{10.3150/16-BEJ857}
Ovcharov, E.~Y.
\newblock {Proper scoring rules and Bregman divergence}.
\newblock \emph{Bernoulli}, 24\penalty0 (1):\penalty0 53 -- 79, 2018.
\newblock \doi{10.3150/16-BEJ857}.
\newblock URL \url{https://doi.org/10.3150/16-BEJ857}.

\bibitem[Paszke et~al.(2019)Paszke, Gross, Massa, Lerer, Bradbury, Chanan,
  Killeen, Lin, Gimelshein, Antiga, Desmaison, Kopf, Yang, DeVito, Raison,
  Tejani, Chilamkurthy, Steiner, Fang, Bai, and Chintala]{NEURIPS2019_9015}
Paszke, A., Gross, S., Massa, F., Lerer, A., Bradbury, J., Chanan, G., Killeen,
  T., Lin, Z., Gimelshein, N., Antiga, L., Desmaison, A., Kopf, A., Yang, E.,
  DeVito, Z., Raison, M., Tejani, A., Chilamkurthy, S., Steiner, B., Fang, L.,
  Bai, J., and Chintala, S.
\newblock Pytorch: An imperative style, high-performance deep learning library.
\newblock In \emph{Advances in Neural Information Processing Systems 32}, pp.\
  8024--8035. Curran Associates, Inc., 2019.

\bibitem[Pedregosa et~al.(2011)Pedregosa, Varoquaux, Gramfort, Michel, Thirion,
  Grisel, Blondel, Prettenhofer, Weiss, Dubourg, Vanderplas, Passos,
  Cournapeau, Brucher, Perrot, and Duchesnay]{scikit-learn}
Pedregosa, F., Varoquaux, G., Gramfort, A., Michel, V., Thirion, B., Grisel,
  O., Blondel, M., Prettenhofer, P., Weiss, R., Dubourg, V., Vanderplas, J.,
  Passos, A., Cournapeau, D., Brucher, M., Perrot, M., and Duchesnay, E.
\newblock Scikit-learn: Machine learning in {P}ython.
\newblock \emph{Journal of Machine Learning Research}, 12:\penalty0 2825--2830,
  2011.

\bibitem[Pfau(2013)]{pfau2013generalized}
Pfau, D.
\newblock A generalized bias-variance decomposition for bregman divergences.
\newblock \emph{Unpublished Manuscript}, 2013.

\bibitem[Rockafellar(1970)]{rockafellar1970convex}
Rockafellar, R.~T.
\newblock \emph{Convex analysis}, volume~18.
\newblock Princeton university press, 1970.

\bibitem[Si et~al.(2009)Si, Tao, and Geng]{si2009bregman}
Si, S., Tao, D., and Geng, B.
\newblock Bregman divergence-based regularization for transfer subspace
  learning.
\newblock \emph{IEEE Transactions on Knowledge and Data Engineering},
  22\penalty0 (7):\penalty0 929--942, 2009.

\bibitem[Song et~al.(2021)Song, Sohl-Dickstein, Kingma, Kumar, Ermon, and
  Poole]{song2021scorebased}
Song, Y., Sohl-Dickstein, J., Kingma, D.~P., Kumar, A., Ermon, S., and Poole,
  B.
\newblock Score-based generative modeling through stochastic differential
  equations.
\newblock In \emph{International Conference on Learning Representations}, 2021.
\newblock URL \url{https://openreview.net/forum?id=PxTIG12RRHS}.

\bibitem[Telgarsky \& Dasgupta(2012)Telgarsky and
  Dasgupta]{telgarsky2012agglomerative}
Telgarsky, M. and Dasgupta, S.
\newblock Agglomerative bregman clustering.
\newblock In \emph{Proceedings of the 29th International Coference on
  International Conference on Machine Learning}, pp.\  1011--1018, 2012.

\bibitem[Tomani \& Buettner(2021)Tomani and Buettner]{tomani2021towards}
Tomani, C. and Buettner, F.
\newblock Towards trustworthy predictions from deep neural networks with fast
  adversarial calibration.
\newblock In \emph{Proceedings of the AAAI Conference on Artificial
  Intelligence}, volume~35, pp.\  9886--9896, 2021.

\bibitem[Wang et~al.(2019)Wang, Li, Aertsen, Deprest, Ourselin, and
  Vercauteren]{wang2019aleatoric}
Wang, G., Li, W., Aertsen, M., Deprest, J., Ourselin, S., and Vercauteren, T.
\newblock Aleatoric uncertainty estimation with test-time augmentation for
  medical image segmentation with convolutional neural networks.
\newblock \emph{Neurocomputing}, 338:\penalty0 34--45, 2019.

\bibitem[Yang et~al.(2020)Yang, Yu, You, Steinhardt, and
  Ma]{yang2020rethinking}
Yang, Z., Yu, Y., You, C., Steinhardt, J., and Ma, Y.
\newblock Rethinking bias-variance trade-off for generalization of neural
  networks.
\newblock In \emph{International Conference on Machine Learning}, pp.\
  10767--10777. PMLR, 2020.

\bibitem[Yen et~al.(2019)Yen, Liu, Hsin, Lin, and Chen]{yen2019application}
Yen, M.-H., Liu, D.-W., Hsin, Y.-C., Lin, C.-E., and Chen, C.-C.
\newblock Application of the deep learning for the prediction of rainfall in
  southern taiwan.
\newblock \emph{Scientific reports}, 9\penalty0 (1):\penalty0 1--9, 2019.

\bibitem[Zalinescu(2002)]{zalinescu2002convex}
Zalinescu, C.
\newblock \emph{Convex analysis in general vector spaces}.
\newblock World scientific, 2002.

\end{thebibliography}
\bibliographystyle{icml2022}
}

%%%%%%%%%%%%%%%%%%%%%%%%%%%%%%%%%%%%%%%%%%%%%%%%%%%%%%%%%%%%
\newpage

\section*{}

\newpage

\appendix

\onecolumn

%\aistatstitle{Uncertainty Estimates of Predictions via a General Bias-Variance Decomposition: \\
%Supplementary Materials}

\section{OVERVIEW}
In this appendix, we provide more rigorous definitions than in the main paper and the missing proofs in Appendix \ref{app:proofs}.
Further, we give a more detailed description of the experiments and showcase additional empirical results in Appendix \ref{app:exp}.

\section{MISSING PROOFS}
\label{app:proofs}

We first provide some more rigorous definitions than in the main paper in Section \ref{sec:app_prel}.
There, we also introduce and prove some basic facts with respect to these definitions, which we will then use in the proofs of Lemma \ref{le:conj_breg} in Section \ref{sec:proof_conj_breg}, Theorem \ref{th:scores_bvd} in Section \ref{sec:proof_scores_bvd}, Proposition \ref{prop:exp_fam_decomp} in Section \ref{sec:proof_exp_fam_decomp}, Corollary \ref{cor:classifnll} in Section \ref{sec:proof_classifnll}, and Proposition \ref{prop:BI_props} in Section \ref{sec:proof_BI_probs}.

\subsection{Preliminaries}
\label{sec:app_prel}

In the following, we introduce definitions and supporting Lemmas to derive our contributions in later sections.

We will make use of the \textbf{convex hull operator} defined as $\mathrm{co} \left( A \right) = \bigcap \left\{ C \subset X \mid A \subset C, C \text{ convex} \right\}$ for a set $A \subset X$ in a real linear vector space $X$ \citep{zalinescu2002convex}.
It consists of all finite convex combinations of elements in $A$.
Since the definition of the convex conjugate in the main paper is rather informal, we also state a more rigorous version \citep{zalinescu2002convex}.

\begin{definition}
    Given a vector space $X$ with dual vector space $X^*$, pairing $\left\langle x^*, x \right\rangle = x \left( x ^* \right)$ for $x \in X$ and $x^* \in X^*$, and a function $\phi \colon X \to \mathbb{R} \cup \left\{- \infty, \infty \right\}$, the convex conjugate $\phi^* \colon X^* \to \mathbb{R} \cup \left\{- \infty, \infty \right\}$ of $\phi$ is defined as $\phi^* \left( x^* \right) = \sup_{x \in X} \left\langle x^*, x \right\rangle - \phi \left( x \right)$.
\end{definition}

In the case $U \neq X$, we follow the convention that a convex $\phi \colon U \to \mathbb{R}$ is extended to $X$ via $\phi \left( x \right) = \infty$ for $x \nin U$ \citep{rockafellar1970convex, zalinescu2002convex}.
We also restate the definition of subgradients in a more formal manner.

\begin{definition}
    Let $\phi \colon U \to \mathbb{R}$ be a convex function in a vector space $X \supset U$ with dual vector space $X^*$ and pairing $\left\langle x^*, x \right\rangle = x \left( x ^* \right)$ for $x \in X$ and $x^* \in X^*$.
    The \textbf{subdifferential} of $\phi$ at a point $x \in U$ is defined as $\partial \phi \left( x \right) = \left\{ x^\prime \in X^* \mid \phi \left( y \right) \geq \phi \left( x \right) + \left\langle x^\prime, y - x \right\rangle, y \in U \right\}$.
    An element $x^\prime \in \partial \phi \left( x \right)$ is called subgradient of $\phi$ at $P$.
    We call a function $\phi^\prime \colon U \to X^*$ defined as $\phi^\prime \left( x \right) = x^\prime \in \partial \phi \left( x \right)$ a \textbf{subgradient} of $\phi$ on $U$.
\label{def:subdiff}
\end{definition}

Note that the inequality for the subgradient becomes strict if $\phi$ is strictly convex and $x \neq y$.
To not confuse a subgradient $x^\prime$ with other elements $x^*$ in the dual vector space $X^*$, we will write it with '$\prime$' instead of '$*$' contrary to other literature.
We did so in the main paper and continue like this throughout the appendix.

Further, as mentioned in Section \ref{sec:prel}, we use the definition of a (restricted) functional Bregman divergence for dual vector spaces based on \cite{10.3150/16-BEJ857}.

\begin{definition}
    Let $\phi \colon U \to \mathbb{R}$ be a convex function in a vector space $X \supset U$ with dual vector space $X^*$ and pairing $\left\langle x^*, x \right\rangle = x \left( x ^* \right)$ for $x \in X$ and $x^* \in X^*$.
    Let $\phi^\prime \colon U \to X^*$ be a subgradient of $\phi$.
    The \textbf{functional Bregman divergence} $d_{\phi, \phi^\prime} \colon U \times U \to \mathbb{R}$ generated by $\left( \phi, \phi^\prime \right)$ is defined as
    \begin{equation}
        d_{\phi, \phi^\prime} \left( x, y \right) = \phi \left( y \right) - \phi \left( x \right) - \left\langle \phi^\prime \left( x \right), y - x \right\rangle.
    \end{equation}
    Let $\phi^\prime_r \colon V \to X^*$ be another subgradient of $\phi$ with $V \subset U$.
    Then, $d_{\phi, \phi^\prime_r} \colon V \times U \to \mathbb{R}$ is a \textbf{restricted} functional Bregman divergence.
\label{def:rfbd}
\end{definition}

In our context, $U$ will be a convex subset of either of two vector spaces, which we introduce from \citep{10.3150/16-BEJ857}.
Let $\mathcal{P}$ be a convex set of probability measures of a measure space $\left( \Omega, \mathcal{F}, \mu \right)$.
We define the space of finite linear combinations of $\mathcal{P}$ as $\mathrm{span} \mathcal{P} = \left\{ \sum_{i=1}^n a_i P_i \mid a_1, \dots, a_n \in \mathbb{R}, P_1, \dots, P_n \in \mathcal{P}, n \in \mathbb{N} \right\}$.
We define the space of $\mathcal{P}$-integrable functions as $\mathcal{L} \left( \mathcal{P} \right) = \left\{ f \mid \int \left\lvert f \right\rvert \mathrm{d} P < \infty, P \in \mathcal{P} \right\}$.
Further, we use $. \cdot . \colon \mathcal{L} \left( \mathcal{P} \right) \times \mathrm{span} \mathcal{P} \to \mathbb{R}$ defined as $f \cdot P = \int f \mathrm{d} P$ as the pairing between the dual spaces $\mathrm{span} \mathcal{P}$ and $\mathcal{L} \left( \mathcal{P} \right)$. \\
When $\phi \colon U \to \mathbb{R}$ is the negative entropy of a proper score, we will have $U = \mathcal{P}$, $X = \mathrm{span} \mathcal{P}$, $X^* = \mathcal{L} \left( \mathcal{P} \right)$, and $\left\langle x^*, x \right\rangle = x^* \cdot x$.
The other case is when $\phi \colon U \to \mathbb{R}$ is the convex conjugate of a negative entropy of a proper score $S$.
Then, we have $U = \mathrm{co} \left( \left\{ S \left( P \right) \mid P \in \mathcal{P} \right\} \right)$, $X = \mathcal{L} \left( \mathcal{P} \right)$, $X^* = \mathrm{span} \mathcal{P}$, and $\left\langle x^*, x \right\rangle = x \cdot x^*$.
To reduce the possibility of confusion, we will not be using a general $U$, $X$, or $X^*$ in the following.
Rather, we only proof the exchange of arguments in functional Bregman divergences as encountered in proper scores \citep{10.3150/16-BEJ857}.
This way, it is always clear if functions have distributions as input or as output.
A proof for more general vector spaces is left for future work.

%We imply this throughout this work.
%The convex conjugate of a function $\phi$ is defined as $\phi^* \left( P^* \right) = \sup_{Q \in \mathrm{span} \mathcal{P}} P^* \cdot Q - \phi \left( Q \right)$.
%Further, $\phi$ will be either the negative entropy of a proper score or its convex conjugate.
Further, note that in contrast to \cite{hendrickson1971proper} and \cite{10.3150/16-BEJ857}, we do \emph{not} extend a score and its entropy to the cone of $\mathcal{P}$ defined as $\mathrm{cone} \left( \mathcal{P} \right) = \left\{ \lambda P \mid \lambda > 0, P \in \mathcal{P} \right\}$.
%The reason for this will become apparent after discussing some further properties.
Doing so would make the entropy a 1-homogeneous function.
But, the convex conjugate of a 1-homogeneous function is always zero \citep{fenchel1949conjugate}.
Consequently, we cannot generate a meaningful Bregman Information with the convex conjugate of an entropy extended to $\mathrm{cone} \left( \mathcal{P} \right)$.

We now state the following, which will be used in later proofs.

\begin{lemma}
    Let $\phi \colon \mathcal{P} \to \mathbb{R}$ be a strictly convex, lower semicontinuous function with subgradient $\phi^\prime$ such that $\phi \left( P \right) = \phi^\prime \left( P \right) \cdot P$ for all $P \in \mathcal{P}$.   
    Then, for all $P^\prime \in \phi^\prime \left( \mathcal{P} \right) \coloneqq \left\{ \phi^\prime \left( P \right) \mid P \in \mathcal{P} \right\}$ we have
    \begin{equation}
        \phi^* \left( P^\prime \right) = 0
    \end{equation}
    and for all $R^* \in \mathrm{co} \left( \phi^\prime \left( \mathcal{P} \right) \right)$ with $R^* \nin \phi^\prime \left( \mathcal{P} \right)$ we have
    \begin{equation}
        -\infty < \phi^* \left( R^* \right) < 0.
    \end{equation}
\label{le:SubgOfConv}
\end{lemma}

\begin{proof}
By definition we have for $P, Q \in \mathcal{P}$ and subgradient $P^\prime$ of $\phi \left( P \right)$ that
\begin{equation}
    \phi \left( Q \right) \geq \phi \left( P \right) + P^\prime \cdot \left(Q - P \right),
\end{equation}
from which follows 
\begin{equation}
    P^\prime \cdot P - \phi \left( P \right) \geq P^\prime \cdot Q - \phi \left( Q \right).
\end{equation}
% with equality if and only if $y = x$.
%I.e. if $p = \partial \phi \left( x \right)$, then $\left( \partial \phi \right)^{-1} \left( p \right) = y$ maximizes the RHS for fixed $x$.
Since $P \in \mathrm{span} \mathcal{P}$, we have
\begin{equation}
    \phi^* \left( P^\prime \right) = \sup_{Q \in \mathcal{P}} P^\prime \cdot Q - \phi \left( Q \right) = P^\prime \cdot P - \phi \left( P \right) = 0.
\end{equation}

As stated in \citep{zalinescu2002convex}, any $R^* \in \mathrm{co} \left( \phi^\prime \left( \mathcal{P} \right) \right)$ can be represented as a convex combination of elements in $\phi^\prime \left( \mathcal{P} \right)$.
To have shorter expressions, we assume $R^*$ is a combination of only two elements $P \neq Q$.
The proof for more combinations is analogous.
We use contradiction to show that the convex conjugate of the convex combination of (two) subgradients is strictly negative.
For this, assume $\phi^* \left( \lambda \phi^\prime \left( P \right) + \left( 1 - \lambda \right) \phi^\prime \left( Q \right) \right) = 0$, then we have

\begin{equation}
\begin{split}
    & \phi^* \left( \lambda \phi^\prime \left( P \right) + \left( 1 - \lambda \right) \phi^\prime \left( Q \right) \right) = 0 \\
    & \implies \sup_{R \in \mathcal{P}} \left( \lambda \phi^\prime \left( P \right) + \left( 1 - \lambda \right) \phi^\prime \left( Q \right) \right) \cdot R - \phi \left( R \right) = 0 \\
    & \implies \sup_{R \in \mathcal{P}} \lambda \left( \phi^\prime \left( P \right) - \phi^\prime \left( R \right) \right) \cdot R + \left( 1 - \lambda \right) \left( \phi^\prime \left( Q \right) - \phi^\prime \left( R \right) \right) \cdot R = 0 \\
    & \implies \sup_{R \in \mathcal{P}} - \lambda d_{\phi, \phi^\prime}\left( P, R \right) - \left( 1 - \lambda \right) d_{\phi, \phi^\prime}\left( Q, R \right) = 0 \\
    & \implies \inf_{R \in \mathcal{P}} d_{\phi, \phi^\prime}\left( P, R \right) + d_{\phi, \phi^\prime}\left( Q, R \right) = 0 \\
    & \implies \exists \left( R_n \right)_n \in \mathcal{P}^{\mathbb{N}} \colon \lim_{n \to \infty} d_{\phi, \phi^\prime}\left( P, R_n \right) = 0 = \lim_{n \to \infty} d_{\phi, \phi^\prime}\left( Q, R_n \right) \\
    \overset{\mathrm{l.s.c.}}&{\implies} \exists \left( R_n \right)_n \in \mathcal{P}^{\mathbb{N}} \colon d_{\phi, \phi^\prime}\left( P, \lim_{n \to \infty} R_n \right) = 0 = d_{\phi, \phi^\prime}\left( Q, \lim_{n \to \infty} R_n \right) \\
    \overset{\text{strictly convex}}&{\implies} \exists \left( R_n \right)_n \in \mathcal{P}^{\mathbb{N}} \colon P = \lim_{n \to \infty} R_n = Q.
\end{split}
\end{equation}

But, the last statement is a contradiction to our assumption $P \neq Q$, subsequently proving our claim.
Further, it must be finite since $ -\infty < - \lambda d_{\phi, \phi^\prime}\left( P, P \right) - \left( 1 - \lambda \right) d_{\phi, \phi^\prime}\left( Q, P \right) \leq \sup_{R \in \mathcal{P}} - \lambda d_{\phi, \phi^\prime}\left( P, R \right) - \left( 1 - \lambda \right) d_{\phi, \phi^\prime}\left( Q, R \right)$.

\end{proof}

From Lemma \ref{le:SubgOfConv} follows that if $\phi$ is the negative entropy of a proper score $S$, then $\phi^* \left( S \left( P \right) \right) = 0$ for all $P \in \mathcal{P}$.
Further, we will make use of the following properties.

\begin{lemma}
    If $\phi \colon \mathcal{P} \to \mathbb{R}$ is strictly convex, then 
    %$\partial \phi \colon U \to U^* \subset \mathcal{N}^\prime$ is bijective, i.e. its inverse 
    any subgradient $\phi^\prime$ of $\phi$ is injective and its inverse $\left(\phi^\prime\right)^{-1}$ exists
    on $\phi^\prime \left( \mathcal{P} \right) \coloneqq \left\{ \phi^\prime \left( P \right) \mid P \in \mathcal{P} \right\}$.
    Further, $\left(\phi^\prime\right)^{-1}$ is a subgradient of the convex conjugate $\phi^*$ on $\phi^\prime \left( \mathcal{P} \right)$.
\label{le:subg_inv}
\end{lemma}

\begin{proof}
The proof is similar to the proof of Theorem 6.2.1 from \cite{kurdila2006convex}.

For $P,Q \in \mathcal{P}$ with $P \neq Q$ we have
\begin{equation}
\begin{split}
    & \phi \left( Q \right) > \phi \left( P \right) + \phi^\prime \left( P \right) \cdot \left( Q - P \right) \\
    \implies & \phi \left( Q \right) - \phi \left( P \right) > \phi^\prime \left( P \right) \cdot \left( Q - P \right) \\
    %& \lambda \phi \left( P \right) + \left( 1 - \lambda \right) \phi \left( Q \right) > \phi \left( \lambda P + \left( 1 - \lambda \right) Q\right) \\
    %\implies & \lambda \left( \phi \left( P \right) - \phi \left( Q \right) \right) > \phi \left( \lambda P + \left( 1 - \lambda \right) Q\right) - \phi \left( Q \right) \\
    %\implies & \phi \left( P \right) - \phi \left( Q \right) > \frac{\phi \left( \lambda P + \left( 1 - \lambda \right) Q\right) - \phi \left( Q \right)}{\lambda} = \frac{\phi \left( Q + \lambda \left( P - Q \right) \right) - \phi \left( Q \right)}{\lambda} \\
    %\overset{\text{Cor} \ref{cor:subfrac}}{\implies} & \phi \left( P \right) - \phi \left( Q \right) > \left\langle \partial \phi \left( Q \right), P - Q \right\rangle_*.
\end{split}
\end{equation}

Reversing the roles of $P$ and $Q$ also gives
\begin{equation}
    \phi \left( P \right) - \phi \left( Q \right) > - \phi^\prime \left( Q \right) \cdot \left( Q - P \right).
\end{equation}

Adding the LHS and RHS of the last two inequalities results in

\begin{equation}
\begin{split}
    & 0 > \left( \phi^\prime \left( P \right) - \phi^\prime \left( Q \right) \right) \cdot \left( Q - P \right) \\
    \implies & \phi^\prime \left( P \right) \neq \phi^\prime \left( Q \right).
\end{split}
\end{equation}

%From this follows that $\partial \phi \colon U \to U^*$ is bijective.
Consequently, since $\phi^\prime \left( P \right)$ is unique for each $P \in \mathcal{P}$, it is injective, and the inverse $\left( \phi^\prime \right)^{-1}$ exists for all $P^\prime \in \phi^\prime \left( \mathcal{P} \right)$.

Next, we show that $\left( \phi^\prime \right)^{-1}$ is a subgradient of $\phi^*$ on $\phi^\prime \left( \mathcal{P} \right)$.
By definition, for all $P^\prime \in \phi^\prime \left( \mathcal{P} \right)$ there exists $P \in \mathcal{P}$ such that $P^\prime = \phi^\prime \left( P \right)$ and $\left( \phi^\prime \right)^{-1} \left( P^\prime \right) = P$.
For all $P^\prime, Q^\prime \in \phi^\prime \left( \mathcal{P} \right)$ we have
\begin{equation}
\begin{split}
    & \phi^* \left( Q^\prime \right) \geq \phi^* \left( P^\prime \right) + \left( Q^\prime - P^\prime \right) \cdot \left( \phi^\prime \right)^{-1} \left( P^\prime \right)\\
    \iff & \sup_{Y \in \mathrm{span} \mathcal{P}} Q^\prime \cdot Y - \phi \left( Y \right) \geq \sup_{X \in \mathrm{span} \mathcal{P}} P^\prime \cdot X - \phi \left( X \right) + Q^\prime \cdot P - P^\prime \cdot P \\
    \overset{\text{Le }\ref{le:SubgOfConv}}{\iff} & Q^\prime \cdot Q - \phi \left( Q \right) \geq P^\prime \cdot P - \phi \left( P \right) + Q^\prime \cdot P - P^\prime \cdot P \\
    \iff & \phi \left( P \right) + Q^\prime \cdot Q \geq \phi \left( Q \right) + Q^\prime \cdot P \\
    \iff & \phi \left( P \right) \geq \phi \left( Q \right) + Q^\prime \cdot \left( P - Q \right). \\
\end{split}
\end{equation}

Since $Q^\prime = \phi^\prime \left( Q \right)$ is a subgradient of $\phi$ at point $Q$, the last line holds and confirms that $\left( \phi^\prime \right)^{-1}$ is a subgradient of $\phi^*$ on $\phi^\prime \left( \mathcal{P} \right)$.

% Last, we show that $U^*$ is convex.
% Assume $P,Q \in U$ and $\lambda \in \left(0, 1\right)$ are arbitrary.
% We will show that there exists $R \in U$ such that $\partial \phi \left( R \right) = \lambda \partial \phi \left( P \right) + \left(1 - \lambda \right) \partial \phi \left( Q \right)$ from which follows the convexity of $U^*$.

% First, we require the function $\Tilde{\phi} \colon \left[0 , 1\right] \to \mathbb{R}$ defined as $\Tilde{\phi} \left( \lambda \right) = \phi \left(\lambda P + \left( 1 - \lambda \right) Q \right)$.
% It is differentiable and strictly convex, since both holds for $\phi$. %$\frac{\mathrm{d}}{\mathrm{d} \lambda} \Tilde{\phi} = \left\langle \partial \phi \left( \lambda P + \left( 1 - \lambda \right) Q \right), P - Q \right\rangle_*$
%$\gamma \lambda_1 + \left(1 - \gamma \right) \lambda_2 P + \left( 1 - \gamma \lambda_1 + \left(1 - \gamma \right) \lambda_2 \right) Q = $
\end{proof}

So far, we established the theoretical foundation to perform the exchange of the arguments in a functional Bregman divergence.
But, Lemma \ref{le:conj_breg} also requires that the convex conjugate is finite on $\mathbb{E} \left[ Q^\prime \right]$.
This is not self evident since $\mathbb{E} \left[ Q^\prime \right] \nin \phi^\prime \left( U \right)$ in general.
Consequently, we also require the following.

\begin{lemma}
    %Let $\partial \phi \left( U \right) \coloneqq \bigcup_{P \in U} \partial \phi \left( P \right)$ be the set of all subgradients of a convex function $\phi \colon U \to \mathbb{R}$.%, and $\left( \Omega, \mathcal{F}, \mathbb{P} \right)$ a probability space with a measureable function $Q^* \colon \Omega \to \partial \phi \left( U \right)$.
    Given a strictly convex function $\phi \colon \mathcal{P} \to \mathbb{R}$ with subgradient $\phi^\prime$ such that $\phi \left( P \right) = \phi^\prime \left( P \right) \cdot P$ for all $P \in \mathcal{P}$,
    and let $Q$ be a 
    %$\partial \phi \left( U \right)$-valued 
    random variable with values in $\mathcal{P}$ 
    %random variable with 
    %such that $\mathbb{E} \left[ Q^\prime \right] \in \mathcal{L} \left( \mathcal{P} \right)$, then 
    such that $\left\lvert \mathbb{E} \left[ \phi^\prime \left( Q \right) \cdot P \right] \right\rvert < \infty$ and $\mathbb{E} \left[ Q^\prime \right] \in \mathrm{co} \left( \phi^\prime \left( \mathcal{P} \right) \right)$, then 
    \begin{equation}
        -\infty < \phi^* \left( \mathbb{E} \left[ \phi^\prime \left( Q \right) \right] \right) \leq 0.
    \end{equation}
    %for any convex function $\phi \colon U \to \mathbb{R}$.
    % we have
    % \begin{equation}
    %     \forall C^* \in \mathrm{Conv} \left( \partial \phi \left( U \right) \right) \colon \quad \phi^* \left( C^* \right) \in \mathbb{R},
    % \end{equation}
    % where $\mathrm{Conv}$ denotes the convex hull operator.
\label{le:conv_conj_E}
\end{lemma}

\begin{proof}
    Let $Q^\prime \coloneqq \phi^\prime \left( Q \right)$.
    
    Since $\mathbb{E} \left[ Q^\prime \right] \in \mathrm{co} \left( \phi^\prime \left( \mathcal{P} \right) \right)$, we have $\phi^* \left( \mathbb{E} \left[ Q^\prime \right] \right) \leq 0$ by Lemma \ref{le:SubgOfConv}.

    Further, due to $\left\lvert \mathbb{E} \left[ \phi^\prime \left( Q \right) \cdot P \right] \right\rvert < \infty$ we have $\mathbb{E} \left[ Q^\prime \right] \in \mathcal{L} \left( \mathcal{P} \right)$.
    Thus, for any $P \in \mathcal{P}$ it holds
    \begin{equation}
        -\infty < \mathbb{E} \left[ Q^\prime \right] \cdot P - \phi \left( P \right) \leq \sup_{Q \in \mathrm{span} \mathcal{P}} \mathbb{E} \left[ Q^\prime \right] \cdot Q - \phi \left( Q \right) = \phi^* \left( \mathbb{E} \left[ Q^\prime \right] \right).
    \end{equation}

    %It follows $\phi^* \left( \mathbb{E} \left[ Q^* \right] \right) \in \mathbb{R}$.

    % \begin{equation}
    % \begin{split}
    %     \phi^* \left( C^* \right) & = \sup_{P \in \mathrm{span} \mathcal{P}} C^* \cdot P - \phi \left( P \right) \\
    %     & = \sup_{P \in \mathrm{span} \mathcal{P}} \left( \lambda P^* + \left(1 - \lambda \right) Q^* \right) \cdot P - \phi \left( P \right) \\
    % \end{split}
    % \end{equation}
\end{proof}

%Further, for $U^* = \left\{ \partial \phi \left( P \right) \mid P \in U \right\}$, $\partial G \colon U \to U^*$ is a bijective function.
%Consequently, $\left( \partial G \right)^{-1} \colon U^* \to U$ exists

We can now offer the missing proofs of the main paper.

\subsection{Proof of Lemma \ref{le:conj_breg}}
\label{sec:proof_conj_breg}

Let $\phi^\prime$ be the subgradient of a strictly convex function $\phi \colon \mathcal{P} \to \mathbb{R}$.
We first show the first equality in Lemma \ref{le:conj_breg}.
Based on the definition of a functional Bregman divergence with $p, q \in \mathcal{P}$, we have
\begin{equation}
\begin{split}
    & d_{\phi, \phi^\prime} \left( p, q \right) \\
    & = \phi \left( q \right) - \phi \left( p \right) - \phi^\prime \left( p \right) \cdot \left( q - p \right) \\
    & = \phi \left( q \right) - \phi \left( p \right) - \phi^\prime \left( p \right) \cdot q + \phi^\prime \left( p \right) \cdot p + \phi^\prime \left( q \right) \cdot q - \phi^\prime \left( q \right) \cdot q \\
    \overset{\text{Le }\ref{le:SubgOfConv}}&{=} \phi^* \left( \phi^\prime \left( p \right) \right) - \phi^* \left( \phi^\prime \left( q \right) \right) - \left( \phi^\prime \left( p \right) - \phi^\prime \left( q \right) \right) \cdot q \\
    \overset{\text{Le }\ref{le:subg_inv}}&{=} \phi^* \left( \phi^\prime \left( p \right) \right) - \phi^* \left( \phi^\prime \left( q \right) \right) - \left( \phi^\prime \left( p \right) - \phi^\prime \left( q \right) \right) \cdot \left( \phi^\prime \right)^{-1} \left( \phi^\prime \left( q \right) \right) \\
    & = d_{\phi^*, \left( \phi^\prime \right)^{-1}} \left( \phi^\prime \left( q \right), \phi^\prime \left( p \right) \right).
\label{eq:fbd_flip}
\end{split}
\end{equation}

Where $d_{\phi^*, \left( \phi^\prime \right)^{-1}} \colon \phi^\prime \left( \mathcal{P} \right) \times \mathrm{co} \left( \phi^\prime \left( \mathcal{P} \right) \right) \to \mathbb{R}$ is well-defined due to Lemma \ref{le:SubgOfConv}.
Since $\mathrm{co} \left( \phi^\prime \left( \mathcal{P} \right) \right)$ is a convex subset in the vector space $\mathcal{L} \left( \mathcal{P} \right)$ and $\left( \phi^\prime \right)^{-1}$ is a subgradient of $\phi^*$ (c.f. Lemma \ref{le:subg_inv}), $d_{\phi^*, \left( \phi^\prime \right)^{-1}}$ is a restricted functional Bregman divergences by Definition \ref{def:rfbd}.
If $p^\prime, q^\prime \in \phi^\prime \left( \mathcal{P} \right)$, we can set $\left(\phi^\prime \right)^{-1} \left( p^\prime \right) = p$ and $\left(\phi^\prime \right)^{-1} \left( q^\prime \right) = q$ in Equation \ref{eq:fbd_flip} and receive a similar result for the second equality in Lemma \ref{le:conj_breg}.

Now, let $Q^\prime$ be a random variable with realizations in $\phi^\prime \left( \mathcal{P} \right)$ such that $\left\lvert \mathbb{E} \left[ Q^\prime \cdot P \right] \right\rvert < \infty$ for all $P \in \mathcal{P}$ and $\mathbb{E} \left[ Q^\prime \right] \in \mathrm{co} \left( \phi^\prime \left( \mathcal{P} \right) \right)$.
Then, we get the last equality in Lemma \ref{le:conj_breg} with  $p^\prime \in \phi^\prime \left( \mathcal{P} \right)$ by
\begin{equation}
\begin{split}
    & \mathbb{E} \left[ d_{\phi^*, \left( \phi^\prime \right)^{-1}} \left( p^\prime, Q^\prime \right) \right] \\
    & = \mathbb{E} \left[ \phi^* \left( Q^\prime \right) - \phi^* \left( p^\prime \right) - \left( Q^\prime - p^\prime \right) \cdot \left( \phi^\prime \right)^{-1} \left( p^\prime \right) \right] \\
    \overset{\text{Le }\ref{le:conv_conj_E}}&{=} \mathbb{E} \left[ \phi^* \left( Q^\prime \right) - \phi^* \left( p^\prime \right) - \left( Q^\prime - p^\prime \right) \cdot \left( \phi^\prime \right)^{-1} \left( p^\prime \right) \right] + \phi^* \left( \mathbb{E} \left[ Q^\prime \right] \right) - \phi^* \left( \mathbb{E} \left[ Q^\prime \right] \right) \\
    & = \phi^* \left( \mathbb{E} \left[ Q^\prime \right] \right) - \phi^* \left( p^\prime \right) - \left( \mathbb{E} \left[ Q^\prime \right] - p^\prime \right) \cdot \left( \phi^\prime \right)^{-1} \left( p^\prime \right) + \mathbb{E} \left[ \phi^* \left( Q^\prime \right) \right] - \phi^* \left( \mathbb{E} \left[ Q^\prime \right] \right) \\
    & = d_{\phi^*, \left( \phi^\prime \right)^{-1}} \left( p^\prime, \mathbb{E} \left[ Q^\prime \right] \right) + \mathbb{B}_{\phi^*} \left[ Q^\prime \right].
\end{split}
\end{equation}

We now have the necessary requirements to prove our main result.

\subsection{Proof of Theorem \ref{th:scores_bvd}}
\label{sec:proof_scores_bvd}

For completeness, we derive the relation between proper scores and functional Bregman divergences in the following, even though it is already known in the literature \citep{10.3150/16-BEJ857}.

Note that a score $S$ proper on $\mathcal{P}$ is a subgradient of $G$ on $\mathcal{P}$ since for all $P, Q \in \mathcal{P}$

\begin{equation}
\begin{split}
    & S \left( Q \right) \cdot Q \geq S \left( P \right) \cdot Q \\
    & \iff S \left( Q \right) \cdot Q \geq S \left( P \right) \cdot P + S \left( P \right) \cdot Q - S \left( P \right) \cdot P \\
    \overset{\text{def}}&{\iff} G \left( Q \right) \geq G \left( P \right) + S \left( P \right) \cdot \left( Q - P \right). \\
\end{split}
\end{equation}

The relation between a proper score $S$ and a functional Bregman divergence $d_{G,S}$ on convex $\mathcal{P}$ is then given by

\begin{equation}
\begin{split}
    S \left( P \right) \cdot Q & = S \left( Q \right) \cdot Q - S \left( Q \right) \cdot Q + S \left( P \right) \cdot Q - S \left( P \right) \cdot P + S \left( P \right) \cdot P \\
    \overset{\text{def}}&{=} G \left( Q \right) - G \left( Q \right) + G \left( P \right) + S \left( P \right) \cdot \left( Q - P \right) \\
    \overset{\text{def}}&{=} G \left( Q \right) - d_{G, S} \left( P, Q \right).
\label{eq:S_to_dGS}
\end{split}
\end{equation}

Now, let $S$ be strictly proper on convex $\mathcal{P}$.
Then its negative entropy $G$ is strictly convex on $\mathcal{P}$ \citep{10.3150/16-BEJ857}.
Further, let $P \colon \Omega \to \mathcal{P}$ be a random variable such that the integrals $\mathbb{E} \left[ S \left( P \right)\left( Y \right) \right]$ and $\mathbb{E} \left[ G \left( P \right) \right]$ exist for all $Y \sim Q \in \mathcal{P}$, and $\mathbb{E} \left[ S \left( P \right) \right] \in \mathrm{co} \left( \phi^\prime \left( \mathcal{P} \right) \right)$.
Then, we have

\begin{equation}
\begin{split}
    \mathbb{E} \left[ - S \left( P \right)\left( Y \right) \right] & = - \mathbb{E} \left[ S \left( P \right) \cdot Q \right] \\
    \overset{\text{Eq }\eqref{eq:S_to_dGS}}&{=} - G \left( Q \right) + \mathbb{E} \left[ d_{G, S} \left( P, Q \right) \right] \\
    \overset{\text{Le }\ref{le:conj_breg}}&{=} - G \left( Q \right) + \mathbb{E} \left[ d_{G^*, S^{-1}} \left( S \left( Q \right), S \left( P \right) \right) \right] \\
    \overset{\text{Le }\ref{le:conj_breg}}&{=} - G \left( Q \right) + d_{G^*, S^{-1}} \left( S \left( Q \right), \mathbb{E} \left[ S \left( P \right) \right] \right) + \mathbb{B}_{G^*} \left[ S \left( P \right) \right]. \\
\end{split}
\end{equation}

\subsection{Proof of Proposition \ref{prop:exp_fam_decomp}}
\label{sec:proof_exp_fam_decomp}

Theorem \ref{th:scores_bvd} is stated for distributions.
Exponential families are usually stated in form of their density or mass function, which are also used for the log-likelihood.
Further, Proposition \ref{prop:exp_fam_decomp} assumes we are restricted to a specific exponential family.
In this context, the Radon-Nikodym derivative of a distribution $P_\theta$ is $p_\theta \coloneqq \frac{\mathrm{d} P_\theta}{\mathrm{d} \mu}$ with base measure $\mu$ of the related measure space $\left( \Omega, \mathcal{F}, \mu \right)$.
For continuous distributions, $\mu$ is the Lebesgue measure, and for discrete distributions, $\mu$ is the counting measure.
We assume the set of distributions $\mathcal{P}$ consists of distributions with the same base measure.
To state our proof for discrete as well as continuous families, we will use the Radon-Nikodym formulation.

Further, we require the more general formulations for the log score, the negative Shannon entropy, and the log partition function by $S \left( P \right) = \ln \frac{\mathrm{d} P}{\mathrm{d} \mu}$, $H \left( P \right) = \ln \frac{\mathrm{d} P}{\mathrm{d} \mu} \cdot P = \int_\Omega \ln \frac{\mathrm{d} P}{\mathrm{d} \mu} \mathrm{d} P$, and $H^* \left( P^* \right) = \ln \int_\Omega \exp P^* \mathrm{d} \mu$ for $P^* \in \mathcal{L} \left( \mathcal{P} \right)$.
For densities, these formulations reduce to the ones provided in Example \ref{ex:log}.
%We will further use the formulations of the negative Shannon entropy $H$ and the log partition function $H^*$ of Example \ref{ex:log}.

For an exponential family, we have 
\begin{equation}
    \frac{\mathrm{d} P_\theta}{\mathrm{d} \mu} \left( x \right) = p_\theta \left( x \right) = \exp \left( \left\langle \theta, T \left( x \right) \right\rangle - A \left( \theta \right) \right) h \left( x \right)
\end{equation}

and, thus, also 

\begin{equation}
    \frac{\mathrm{d} P_\theta}{\mathrm{d} \mu} = p_\theta = \exp \left( \left\langle \theta, T \right\rangle - A \left( \theta \right) \right) h \in \mathcal{L} \left( P \right).
\end{equation}

The last statement also introduces the notation we will use.

For the log score, it follows

% \begin{equation}
%     \ln \frac{\mathrm{d} P_\theta}{\mathrm{d} \mu} = \ln p_{\hat{\theta}} = \left\langle \hat{\theta}, T \right\rangle - A \left( \hat{\theta} \right) - \ln h.
% \end{equation}

% Also, 

\begin{equation}
    \mathbb{E} \left[ S \left( P_{\hat{\theta}} \right) \right] = \mathbb{E} \left [ \ln \frac{\mathrm{d} P_{\hat{\theta}}}{\mathrm{d} \mu} \right] = \left\langle \mathbb{E} \left [ \hat{\theta} \right], T \right\rangle - \mathbb{E} \left [ A \left( \hat{\theta} \right) \right] - \ln h
\label{eq:log_score_def}
\end{equation}
 
 which gives 
 
\begin{equation}
\begin{split}
    H^* \left( \mathbb{E} \left [ \ln \frac{\mathrm{d} P_{\hat{\theta}}}{\mathrm{d} \mu} \right] \right) & = \ln \int \exp \left( \left\langle \mathbb{E} \left [ \hat{\theta} \right], T \right\rangle - \mathbb{E} \left [ A \left( \hat{\theta} \right) \right] \right) h \mathrm{d} \mu \\
    & = \ln \int \exp \left( \left\langle \mathbb{E} \left [ \hat{\theta} \right], T \right\rangle \right) h \mathrm{d} \mu - \mathbb{E} \left [ A \left( \hat{\theta} \right) \right] \\
    & = A \left( \mathbb{E} \left [ \hat{\theta} \right] \right) - \mathbb{E} \left [ A \left( \hat{\theta} \right) \right].
\label{eq:H*=JG}
\end{split}
\end{equation}

Consequently, with $\mathbb{B}_{H^*} \left [ \ln \frac{\mathrm{d} P_{\hat{\theta}}}{\mathrm{d} \mu} \right] = - H^* \left( \mathbb{E} \left [ \ln \frac{\mathrm{d} P_{\hat{\theta}}}{\mathrm{d} \mu} \right] \right)$ stated in Example \ref{ex:log} we can already say

\begin{equation}
\begin{split}
    \mathbb{B}_{H^*} \left [ \ln \frac{\mathrm{d} P_{\hat{\theta}}}{\mathrm{d} \mu} \right] & = \mathbb{E} \left [ A \left( \hat{\theta} \right) \right] -  A \left( \mathbb{E} \left [ \hat{\theta} \right] \right) = \mathbb{B}_A \left[ \hat{\theta} \right].
\end{split}
\end{equation}

The reduction from a functional Bregman Information to a vector-based Bregman Information is a remarkable fact for exponential families, which will not hold for the bias term as we will see in the following.
For the bias term, we first have to make some further additional statements.
Note that from definition of the log score, it follows that $S^{-1} \left( P^\prime \right) = \int \exp P^\prime \mathrm{d} \mu$, which is a mapping from $\mathcal{F}$ to $\mathbb{R}$.
To confirm the inverse, note that for all $P \in \mathcal{P}$ and for all $F \in \mathcal{F}$, we have

\begin{equation}
    S^{-1} \left( S \left( P \right) \right) \left( F \right) = \left( \int \exp \ln \frac{\mathrm{d} P}{\mathrm{d} \mu} \mathrm{d} \mu \right) \left( F \right) =  \int_F \frac{\mathrm{d} P}{\mathrm{d} \mu} \mathrm{d} \mu \overset{\text{(i)}}{=} P \left( F \right),
\label{eq:SinvS}
\end{equation}

where we used the Radon-Nikodym theorem in (i).

%Further, for any $x \in \Omega$ let $B_n \left( x \right) \in \mathcal{F}$ be a contracting sequence of events for $n \in \mathbb{N}$ such that $\left\{ x \right\} \subset \cdots \subset B_{n+1} \left( x \right) \subset B_n \left( x \right)$, $\mu \left( B_n \left( x \right) \right) > 0$ and $\lim_{n \to \infty} \mu \left( B_n \left( x \right) \right) = \mu \left( \left\{ x \right\} \right)$.
Then, for all $P^\prime \in \left\{ S \left( P \right) \mid P \in \mathcal{P} \right\} \subset \mathcal{L} \left( \mathcal{P} \right)$ and $x \in \Omega$ it holds almost surely
%Then, for all $P^* \in \mathcal{L} \left( \mathcal{P} \right)$ such that $\left\lvert \int_\Omega \exp P^* \mathrm{d} \mu \right\rvert < \infty$ and $x \in \Omega$ it holds almost surely

\begin{equation}
\begin{split}
    S \left( S^{-1} \left( P^\prime \right) \right) \left( x \right) & = \left( \ln \frac{\mathrm{d} \int \exp P^\prime \mathrm{d} \mu}{\mathrm{d} \mu} \right) \left( x \right) \\
    & = \ln \left( \frac{\mathrm{d} \int \exp P^* \mathrm{d} \mu}{\mathrm{d} \mu} \left( x \right) \right) \\
% & = \ln \left( \lim_{n \to \infty} \frac{1}{\mu \left( B_n \left( x \right) \right)} \int_{B_n \left( x \right)} \exp P^* \mathrm{d} \mu \right) \\
    & = \ln \left( \lim_{B \to \left\{ x \right\}} \frac{1}{\mu \left( B \right)} \int_B \exp P^\prime \mathrm{d} \mu \right) \\
    \overset{\text{(ii)}}&{=} \ln \left( \exp P^\prime \left( x \right) \right) = P^\prime \left( x \right),
\end{split}
\end{equation}

where we used the Lebesgue differentiation theorem in (ii).

\begin{remark}
It is possible to extend $S^{-1}$ via $\frac{1}{\int_\Omega \exp P^* \mathrm{d} \mu} S^{-1}$ to $\mathcal{L} \left( \mathcal{P} \right)$.
This makes it the subgradient of $H^*$ on $\mathcal{L} \left( \mathcal{P} \right)$ but it is unnecessary for the proof.
\end{remark}

Following from Equation \eqref{eq:log_score_def}, we will also make use of

\begin{equation}
    \int_\Omega \mathbb{E} \left[ \ln \frac{\mathrm{d} P_{\hat{\theta}}}{\mathrm{d} \mu} \right] \mathrm{d} Q = \int_\Omega \left\langle \mathbb{E} \left [ \hat{\theta} \right], T \right\rangle - \mathbb{E} \left [ A \left( \hat{\theta} \right) \right] - \ln h \; \mathrm{d} Q \\
    \overset{Y \sim Q}{=} \left\langle \mathbb{E} \left [ \hat{\theta} \right], \mathbb{E} \left[ T \left( Y \right) \right] \right\rangle - \mathbb{E} \left [ A \left( \hat{\theta} \right) \right] - \mathbb{E} \left[ \ln h \left( Y \right) \right].
\label{eq:e_ln_q}
\end{equation}

For the bias term in Theorem \ref{th:scores_bvd}, we first assume a general $Y \sim Q$ to demonstrate our claim about Proposition \ref{prop:exp_fam_decomp} that the decomposition holds even when the distribution assumption is wrong.
Now, we can state that

\begin{equation}
\begin{split}
    & d_{H^*, S^{-1}} \left( \ln \frac{\mathrm{d} Q}{\mathrm{d} \mu}, \mathbb{E} \left[ \ln \frac{\mathrm{d} P_{\hat{\theta}}}{\mathrm{d} \mu} \right] \right) \\
    \overset{\text{def}}&{=} \underbrace{H^* \left( \mathbb{E} \left[ \ln \frac{\mathrm{d} P_{\hat{\theta}}}{\mathrm{d} \mu} \right] \right)}_{\overset{\text{Eq }\eqref{eq:H*=JG}}{=} A \left( \mathbb{E} \left [ \hat{\theta} \right] \right) - \mathbb{E} \left [ A \left( \hat{\theta} \right) \right]} - \underbrace{H^* \left( \ln \frac{\mathrm{d} Q}{\mathrm{d} \mu} \right)}_{\overset{\text{Ex }\ref{ex:log}}{=}0} - \left( \mathbb{E} \left[ \ln \frac{\mathrm{d} P_{\hat{\theta}}}{\mathrm{d} \mu} \right] - \ln \frac{\mathrm{d} Q}{\mathrm{d} \mu} \right) \cdot \underbrace{S^{-1} \left( \ln \frac{\mathrm{d} Q}{\mathrm{d} \mu} \right)}_{\overset{\text{Eq }\eqref{eq:SinvS}}{=} Q} \\
    & = A \left( \mathbb{E} \left [ \hat{\theta} \right] \right) - \mathbb{E} \left [ A \left( \hat{\theta} \right) \right] - \int_\Omega \mathbb{E} \left[ \ln \frac{\mathrm{d} P_{\hat{\theta}}}{\mathrm{d} \mu} \right] \mathrm{d} Q + H \left( Q \right) \\
    \overset{\text{Eq }\eqref{eq:e_ln_q}}&{=} A \left( \mathbb{E} \left [ \hat{\theta} \right] \right) - \left\langle \mathbb{E} \left [ \hat{\theta} \right], \mathbb{E} \left[ T \left( Y \right) \right] \right\rangle + \mathbb{E} \left[ \ln h \left( Y \right) \right] + H \left( Q \right)\\
    & = A \left( \mathbb{E} \left [ \hat{\theta} \right] \right) - \left\langle \mathbb{E} \left [ \hat{\theta} \right], \mathbb{E} \left[ T \left( Y \right) \right] \right\rangle + \mathbb{E} \left[ \ln h \left( Y \right) \right] + H \left( Q \right) + A^* \left( \mathbb{E} \left[ T \left( Y \right) \right] \right) - A^* \left( \mathbb{E} \left[ T \left( Y \right) \right] \right) \\
    \overset{\text{Le }\ref{le:SubgOfConv}}&{=} A \left( \mathbb{E} \left [ \hat{\theta} \right] \right) - \left\langle \mathbb{E} \left [ \hat{\theta} \right], \mathbb{E} \left[ T \left( Y \right) \right] \right\rangle + \mathbb{E} \left[ \ln h \left( Y \right) \right] + H \left( Q \right) + \\
    & \quad \quad + \left\langle \nabla A \left( \mathbb{E} \left[ T \left( Y \right) \right] \right), \mathbb{E} \left[ T \left( Y \right) \right] \right\rangle - A \left( \nabla A^* \left( \mathbb{E} \left[ T \left( Y \right) \right] \right) \right) - A^* \left( \mathbb{E} \left[ T \left( Y \right) \right] \right) \\
    & = A \left( \mathbb{E} \left [ \hat{\theta} \right] \right) - A \left( \nabla A^* \left( \mathbb{E} \left[ T \left( Y \right) \right] \right) \right) - \left\langle \nabla A \left( \nabla A^* \left( \mathbb{E} \left[ T \left( Y \right) \right] \right) \right), \mathbb{E} \left [ \hat{\theta} \right] - A^* \left( \mathbb{E} \left[ T \left( Y \right) \right] \right) \right\rangle + \\
    & \quad \quad + \mathbb{E} \left[ \ln h \left( Y \right) \right] + H \left( Q \right) - A^* \left( \mathbb{E} \left[ T \left( Y \right) \right] \right) \\
    \overset{\text{def}}&{=} d_A \left( \nabla A^* \left( \mathbb{E} \left[ T \left( Y \right) \right] \right), \mathbb{E} \left [ \hat{\theta} \right] \right) + \mathbb{E} \left[ \ln h \left( Y \right) \right] + H \left( Q \right) - A^* \left( \mathbb{E} \left[ T \left( Y \right) \right] \right).
\label{eq:bias_red}
\end{split}
\end{equation}

As we can see, while the functional Bregman Information nicely reduces to a vector-based Bregman Information, it is not the case for the functional form of the bias.
Specifically, the functional bias and noise term have to be taken together to end up with a vector-based bias term.

So far, $Y$ was arbitrarily distributed, but we require an additional restriction to end up with the formulation in Proposition \ref{prop:exp_fam_decomp}.
If we assume that $Y \sim Q = P_\theta$ follows a distribution from the respective exponential family with natural parameter $\theta$, then we have $\nabla A^* \left( \mathbb{E} \left[ T \left( Y \right) \right] \right) = \theta$, which gives in the last line in Equation \eqref{eq:bias_red} that

\begin{equation}
    d_{H^*, S^{-1}} \left( \ln \frac{\mathrm{d} Q}{\mathrm{d} \mu}, \mathbb{E} \left[ \ln \frac{\mathrm{d} P_{\hat{\theta}}}{\mathrm{d} \mu} \right] \right) = d_A \left( \theta, \mathbb{E} \left [ \hat{\theta} \right] \right) + \mathbb{E} \left[ \ln h \left( Y \right) \right] + H \left( Q \right) - A^* \left( \nabla A \left( \theta \right) \right).
\end{equation}

\subsection{Proof of Corollary \ref{cor:classifnll}}
\label{sec:proof_classifnll}

We now provide proof for the closed-form decomposition of the classification log-likelihood.
Since it corresponds to the log-likelihood for the categorical distribution (an exponential family), we can directly derive it from Proposition \ref{prop:exp_fam_decomp}.

For the categorical distribution with $k$ classes, we have for $\theta \in \Theta = \mathbb{R}^{k-1}$ the log-partition $A \left( \theta \right) = \ln \left( 1 + \sum_i^{k-1} \exp \theta_i \right)$ and $h \equiv 1$.
The gradient is $\nabla A \left( \theta \right) = \frac{1}{1 + \sum_i^{k-1} \exp \theta_i} \left( \exp \theta_1, \dots, \exp \theta_{k-1} \right)^\intercal$ %for which $\mathrm{sm} \left( z \right) = \left( \frac{1}{1 + \sum_i^{k-1} \exp \theta_i}, \frac{\partial}{\partial z_1} A \left( \theta \right), \dots, \theta_{k-1} \right)^\intercal$
Further, we have for $\theta^* \in \Theta^* = \left\{ \left( p_1, \dots, p_{k-1} \right)^\intercal \mid p_1, \dots, p_{k-1} \in \left(0, 1 \right), \sum_i p_i < 1 \right\}$ the convex conjugate $A^* \left( \theta^* \right) = \left( 1 - \sum_{i=1}^{k-1} \theta^*_i \right) \ln \left( 1 - \sum_{i=1}^{k-1} \theta^*_i \right) + \sum_{i=1}^{k-1} \theta^*_i \ln \theta^*_i$ with $\nabla A^* \left( \theta^* \right) = \left( \ln \frac{\theta^*_1}{1 - \sum_{i=1}^{k-1} \theta^*_i }, \dots, \ln \frac{\theta^*_{k-1}}{1 - \sum_{i=1}^{k-1} \theta^*_i } \right)^\intercal$.

% Further, we will identify every $\theta \in \Theta$ with a $z \in \mathbb{R}^{k-1} \times \left\{ 0 \right\}$ via $z = \left( \theta_1, \dots, \theta_{k-1}, 0 \right)^\intercal$ and every $\theta^* \in \Theta^*$ to a probability vector $p = \left( \theta^*_1, \dots, \theta^*_{k-1}, 1 - \sum_{i=1}^{k-1} \theta^*_i \right)^\intercal \in \mathbb{R}^k$.
% This relates the LogSumExp function to $A$ via $\mathrm{LSE} \left( z \right) = \log \sum_i \exp z_i = A \left( \theta \right)$, 
% the softmax function to 
%$\nabla A$ via $\mathrm{sm} \left( z \right) = \left( \frac{\partial}{\partial \theta_1} A \left( \theta \right), \dots, \frac{\partial}{\partial \theta_{k-1}} A \left( \theta \right), 1 - \sum_{i=1}^{k-1} \frac{\partial}{\partial \theta_i} A \left( \theta \right) \right)^\intercal$, 
% $\nabla A$ via $\mathrm{sm} \left( z \right) = \left( \nabla_1 A \left( \theta \right), \dots, \nabla_{k-1} A \left( \theta \right), 1 - \sum_{i=1}^{k-1} \nabla_i A \left( \theta \right) \right)^\intercal$, 
% the Shannon entropy to $A^*$ via $H \left( p \right) = - A^* \left( \theta^* \right)$,
% and the logit function (inverse softmax) to $\nabla A^*$ via %$\mathrm{sm}^{-1} \left( p \right) = \left( \frac{\partial}{\partial \theta^*_1} A^* \left( \theta \right), \dots, \frac{\partial}{\partial \theta^*_1} A^* \left( \theta \right), 0 \right)^\intercal$.
% $\mathrm{sm}^{-1} \left( p \right) = \left( \nabla_1 A^* \left( \theta^* \right), \dots, \nabla_{k-1} A^* \left( \theta^* \right), 0 \right)^\intercal$.
Further, we will relate each $\theta \in \mathbb{R}^{k-1}$ to an equivalence class 
\begin{equation}
    \left[ \theta \right] \coloneqq \left\{ z \in \mathbb{R}^k \mid z_1 = \theta_1 + z_k, \dots, z_{k-1} = \theta_{k-1} + z_k \right\} = \left\{ z \in \mathbb{R}^k \mid \mathrm{sm} \left( \left( \theta_1, \dots, \theta_{k-1}, 0 \right)^\intercal \right) = \mathrm{sm} \left( z \right) \right\}.
\end{equation}
All members of an equivalence class give the same softmax output.
%Note that for all $v \in \left[ \theta \right]$ we have $v - v_k \in \left[ \theta \right]$ and for an additional $w \in \left[ \theta \right]$ also $\mathrm{sm} \left( v \right) = \mathrm{sm} \left( w \right)$.
Now, for any $\theta, \hat{\theta} \in \Theta$ and $z \in \left[ \theta \right], \hat{z} \in \left[ \hat{\theta} \right]$, it holds that

\begin{equation}
\begin{split}
    \mathbb{B}_A \left[ \hat{\theta} \right] & = \mathbb{E} \left[ A \left( \hat{\theta} \right) \right] - A \left( \mathbb{E} \left[ \hat{\theta} \right] \right) \\
    & = \mathbb{E} \left[ \ln \left( 1 + \sum_{i=1}^{k-1} \exp \hat{\theta}_i \right) \right] - \ln \left( 1 + \sum_{i=1}^{k-1} \exp \mathbb{E} \left[ \hat{\theta}_i \right] \right) \\
    & = \mathbb{E} \left[ \ln \left( 1 + \sum_{i=1}^{k-1} \exp \hat{\theta}_i \right) \right] - \ln \left( 1 + \sum_{i=1}^{k-1} \exp \mathbb{E} \left[ \hat{\theta}_i \right] \right) + \mathbb{E} \left[ \ln \exp \hat{z}_k \right] - \ln \exp \mathbb{E} \left[ \hat{z}_k \right] \\
    & = \mathbb{E} \left[ \ln \left( \exp \hat{z}_k + \sum_{i=1}^{k-1} \exp \left( \hat{\theta}_i + \hat{z}_k \right) \right) \right] - \ln \left( \exp \mathbb{E} \left[ \hat{z}_k \right] + \sum_{i=1}^{k-1} \exp \mathbb{E} \left[ \hat{\theta}_i + \hat{z}_k \right] \right) \\
    & = \mathbb{E} \left[ \ln \sum_{i=1}^k \exp \hat{z}_i \right] - \ln \sum_{i=1}^k \exp \mathbb{E} \left[ \hat{z}_i \right] \\
    & = \mathbb{B}_{\mathrm{LSE}} \left[ \hat{z} \right].
\label{eq:BLSE}
\end{split}
\end{equation}

%Let $\mathcal{S}^k$ be the $k$-dimensional simplex and $\mathrm{alr} \colon \mathcal{S}^k \to \mathbb{R}^{k-1}$ the additive log ratio transformation defined as $\mathrm{alr} \left( p \right) = \left( \ln \frac{p_1}{p_k}, \dots, \ln \frac{p_{k-1}}{p_k} \right)^\intercal$ \citep{egozcue2003isometric}.
%Note that we have $\left( \nabla A \right)^{-1} \left( \left(p_1, \dots, p_{k-1} \right)^\intercal \right) = \mathrm{alr} \left( p \right)$.
For the bias term, it holds that

\begin{equation}
\begin{split}
    d_A \left( \theta, \mathbb{E} \left[ \hat{\theta} \right] \right) \overset{\text{def}}&{=} \ln \left( 1 + \sum_{i=1}^{k-1} \exp \mathbb{E} \left[ \hat{\theta}_i \right] \right) - \ln \left( 1 + \sum_{i=1}^{k-1} \exp \theta_i \right) - \sum_{i=1}^{k-1} \frac{\exp \theta_i}{1 + \sum_{j=1}^{k-1} \exp \theta_j} \left( \mathbb{E} \left[ \hat{\theta}_i \right] - \theta_i \right) \\
    & = \ln \left( 1 + \sum_{i=1}^{k-1} \exp \mathbb{E} \left[ \hat{z}_i - \hat{z}_k \right] \right) - \ln \sum_{i=1}^{k} \exp z_i - \sum_{i=1}^{k-1} \frac{\exp z_i}{\sum_{j=1}^{k} \exp z_j} \left( \mathbb{E} \left[ \hat{z}_i - \hat{z}_k \right] - z_i \right) \\
    & = \ln \sum_{i=1}^{k} \exp \mathbb{E} \left[ \hat{z}_i \right] - \mathbb{E} \left[ \hat{z}_k \right] - \ln \sum_{i=1}^{k} \exp z_i - \sum_{i=1}^{k} \mathrm{sm}_i \left( z \right) \left( \mathbb{E} \left[ \hat{z}_i \right] - z_i \right) + \underbrace{\sum_{i=1}^k \mathrm{sm}_i \left( z \right)}_{=1} \mathbb{E} \left[ \hat{z}_k \right] \\
    & = \ln \sum_{i=1}^{k} \exp \mathbb{E} \left[ \hat{z}_i \right] - \ln \sum_{i=1}^{k} \exp z_i - \sum_{i=1}^{k} \mathrm{sm}_i \left( z \right) \left( \mathbb{E} \left[ \hat{z}_i \right] - z_i \right) \\
    %& = \ln \sum_{i=1}^{k} \exp \mathbb{E} \left[ \hat{z}_i \right] - \ln \sum_{i=1}^{k} \exp \mathrm{sm}^{-1}_i \left( Q \right) - \sum_{i=1}^{k} \mathrm{sm}_i \left( z \right) \left( \mathbb{E} \left[ \hat{z}_i \right] - z_i \right) \\
    & = \mathrm{LSE}  \left( \mathbb{E} \left[ \hat{z} \right] \right) - \mathrm{LSE} \left( z \right) - \left\langle \nabla \mathrm{LSE} \left( z \right), \mathbb{E} \left[ \hat{z} \right] - z \right\rangle \\
    \overset{\text{def}}&{=} d_{\mathrm{LSE}} \left( z, \mathbb{E} \left[ \hat{z} \right] \right).
\label{eq:dLSE}
\end{split}
\end{equation}

For $i \in \left\{1, \dots, k \right\}$, we use the probability mass function $Q_i \coloneqq \frac{\mathrm{d}Q}{\mathrm{d}\mu} \left( i \right)$ of the distribution $Q$ with counting measure $\mu$ for shorter notations.
Last, the noise term gives

\begin{equation}
\begin{split}
    - A^* \left( \nabla A \left( \theta \right) \right) = - A^* \left( \left(Q_1, \dots, Q_{k-1} \right)^\intercal \right) = - \sum_{i=1}^{k-1} Q_i \ln Q_i - \left( 1 - \sum_{i=1}^{k-1} Q_i \right) \ln \left( 1 - \sum_{i=1}^{k-1} Q_i \right) = H \left( Q \right).
\label{eq:HLSE}
\end{split}
\end{equation}

Let $\mathrm{sm}^{-1} \left( p \right) \coloneqq \left( \ln \frac{p_1}{p_k}, \dots, \ln \frac{p_{k-1}}{p_k}, 0 \right)^\intercal$ for a probability vector $p$.
Further, let $Q$ have the natural parameter vector $\theta$, which gives $Q = \mathrm{sm} \left( z \right)$ for $z \in \left[ \theta \right]$ and $\mathrm{sm}^{-1} \left( Q \right) \in \left[ \theta \right]$.
Using the Equations \eqref{eq:BLSE}, \eqref{eq:dLSE}, and \eqref{eq:HLSE} with Corollary \ref{cor:classifnll}, we then receive for $Y \sim Q$ and $z \in \left[ \theta \right], \hat{z} \in \left[ \hat{\theta} \right]$

\begin{equation}
\begin{split}
    \mathbb{E} \left[ - \ln \mathrm{sm}_Y \left( \hat{z} \right) \right]
    & = \mathbb{E} \left[ - \ln p_{\hat{\theta}} \left( Y \right) \right] \\
    \overset{\text{Cor }\ref{cor:classifnll}}&{=} - A^* \left( \nabla A \left( \theta \right) \right) - \mathbb{E} \left[ \ln h \left( Y \right) \right] + d_A \left( \theta, \mathbb{E} \left[ \hat{\theta} \right] \right) + \mathbb{B}_{A} \left[ \hat{\theta} \right] \\
    & = H \left( \mathrm{sm} \left( z \right) \right) - \mathbb{E} \left[ \ln 1 \right] + d_{\mathrm{LSE}} \left( z, \mathbb{E} \left[ \hat{z} \right] \right) + \mathbb{B}_{\mathrm{LSE}} \left[ \hat{z} \right] \\
    & = H \left( Q \right) + d_{\mathrm{LSE}} \left( \mathrm{sm}^{-1} \left( Q \right), \mathbb{E} \left[ \hat{z} \right] \right) + \mathbb{B}_{\mathrm{LSE}} \left[ \hat{z} \right]. \\
\end{split}
\end{equation}

\subsection{Proof of Proposition \ref{prop:BI_props}}
\label{sec:proof_BI_probs}

We prove each property in the following.
The arguments are constructed in a generality such that the functional case is always covered.

\subsubsection{General law of total variance}

Let $\phi \colon U \to \mathbb{R}$ be a convex function on a convex subset $U$ of a vector space.
This includes the case of a vector space consisting of functions.
Assume that $X$ and $Y$ are random variables, where $X$ has observations in $U$. If $\mathbb{E} \left[ \mathbb{E} \left[ \phi \left( X \right) \mid Y \right] \right]$ exists, then by Tonelli's theorem and Jensen's inequality the other integrals in the following also exist and we have

\begin{equation}
\begin{split}
    & \mathbb{E} \left[ \mathbb{B}_\phi \left[ X \mid Y \right] \right] + \mathbb{B}_\phi \left[ \mathbb{E} \left[ X \mid Y \right] \right] \\
    & = \mathbb{E} \left[ \mathbb{E} \left[ \phi \left( X \right) \mid Y \right] - \phi \left( \mathbb{E} \left[ X \mid Y \right] \right) \right] + \mathbb{E} \left[ \phi \left( \mathbb{E} \left[ X \mid Y \right] \right) \right] - \phi \left( \mathbb{E} \left[ \mathbb{E} \left[ X \mid Y \right] \right] \right) \\
    & = \mathbb{E} \left[ \mathbb{E} \left[ \phi \left( X \right) \mid Y \right] \right] - \phi \left( \mathbb{E} \left[ \mathbb{E} \left[ X \mid Y \right] \right] \right) \\
    & = \mathbb{E} \left[ \phi \left( X \right) \right] - \phi \left( \mathbb{E} \left[ X \right] \right) \\
    & = \mathbb{B}_\phi \left[ X \right].
\end{split}
\end{equation}

\subsubsection{Proof of Equation \ref{eq:nv_ens}}
\label{sec:proof_nv_ens}

Let $\phi$ be a convex function in a vector space and $X_1, \dots, X_{2^n}$ i.i.d. random variables such that $\mathbb{E} \left[ \phi \left( X_1 \right) \right]$ exists.
Since $\phi \left( \mathbb{E} \left[ 2^{-n+1} \sum_{i=1}^{2^{n-1}} X_i \right] \right) = \phi \left( \mathbb{E} \left[ 2^{-n} \sum_{i=1}^{2^n} X_i \right] \right)$ due to i.i.d. assumption, we only have to show $\mathbb{E} \left[ \phi \left( 2^{-n} \sum_{i=1}^{2^n} X_i \right) \right] < \mathbb{E} \left[ \phi \left( 2^{-n+1} \sum_{i=1}^{2^{n-1}} X_i \right) \right]$.
We do this by using Jensen's inequality for strict convexity:

\begin{equation}
\begin{split}
    \mathbb{E} \left[ \phi \left( 2^{-n} \sum_{i=1}^{2^n} X_i \right) \right] & = \mathbb{E} \left[ \phi \left( \frac{1}{2} 2^{-n+1} \sum_{i=1}^{2^{n-1}} X_i + \frac{1}{2} 2^{-n+1} \sum_{i=2^{n-1} + 1}^{2^n} X_i \right) \right] \\
    & < \mathbb{E} \left[ \frac{1}{2} \phi \left( 2^{-n+1} \sum_{i=1}^{2^{n-1}} X_i \right) + \frac{1}{2} \phi \left( 2^{-n+1} \sum_{i=2^{n-1} + 1}^{2^n} X_i \right) \right] \\
    & = \frac{1}{2} \mathbb{E} \left[ \phi \left( 2^{-n+1} \sum_{i=1}^{2^{n-1}} X_i \right) \right] + \frac{1}{2} \mathbb{E} \left[ \phi \left( 2^{-n+1} \sum_{i=2^{n-1} + 1}^{2^n} X_i \right) \right] \\
    \overset{\mathrm{iid}}&{=} \mathbb{E} \left[ \phi \left( 2^{-n+1} \sum_{i=1}^{2^{n-1}} X_i \right) \right].
\end{split}
\end{equation}
 
In combination with the definition of Bregman Information follows the statement in Equation \ref{eq:nv_ens}.

\subsubsection{Limit case}

Let $\phi$ and $X_1, \dots, X_n$ be defined as in the previous proof with finite mean $\mathbb{E} \left[ X_1 \right]$.
Additionally, $\phi$ is almost surely continuous.
Due to the definition of Bregman Information, we only have to show that

\begin{equation}
    \lim_{n \to \infty} \phi \left( \frac{1}{n} \sum_{i=1}^n X_i \right) \overset{\text{a.s.}}{=} \phi \left( \mathbb{E} \left[ X_1 \right] \right).
\label{eq:phi_mean_conv_as}
\end{equation}

Theorem 8.32 in \citep{capinski2004measure} gives $\lim_{n \to \infty} \frac{1}{n} \sum_{i=1}^n X_i \overset{\text{a.s.}}{=} \mathbb{E} \left[ X_1 \right]$.
Note that we have in general for any random variable $X$ with finite mean that $\left\{ \omega \in \Omega \mid X \left( \omega \right) = \mathbb{E} \left[ X \right] \right\} \subset \left\{ \omega \in \Omega \mid \phi \left( X \left( \omega \right) \right) = \phi \left( \mathbb{E} \left[ X \right] \right) \right\}$.
It follows with the initial conditions that
\begin{equation}
\begin{split}
    1 & = \mathbb{P} \left( \left\{ \omega \in \Omega \mid \lim_{n \to \infty} \frac{1}{n} \sum_{i=1}^n X_i \left( \omega \right) = \mathbb{E} \left[ X_1 \right] \right\} \right) \\
    & \leq \mathbb{P} \left( \left\{ \omega \in \Omega \mid \phi \left( \lim_{n \to \infty} \frac{1}{n} \sum_{i=1}^n X_i \left( \omega \right) \right) = \phi \left( \mathbb{E} \left[ X_1 \right] \right) \right\} \right) \\
    & = \mathbb{P} \left( \left\{ \omega \in \Omega \mid \lim_{n \to \infty} \phi \left( \frac{1}{n} \sum_{i=1}^n X_i \left( \omega \right) \right) = \phi \left( \mathbb{E} \left[ X_1 \right] \right) \right\} \right) \leq 1.    
\end{split}
\end{equation}

Consequently, Equation \eqref{eq:phi_mean_conv_as} holds and with it the statement $\lim_{n \to \infty} \mathbb{B}_\phi \left[ \frac{1}{n} \sum_{i=1}^n X_i \right] \overset{\text{a.s.}}{=} 0$.

\section{EXTENDED EXPERIMENTS}
\label{app:exp}

In this section, we give additional details to the experiments in the main paper, and also provide further results of extended experiments.
In Section \ref{sec:more_sims}, we conduct additional simulation studies to compare common classifiers in terms of their Bregman Information similar to Figure \ref{fig:BI_toy_0} and \ref{fig:BI_approx_ext}.
Further, we investigate our proposed Bregman Information threshold algorithm in more detail on CIFAR-10 (-C) and ImageNet (-C) in Section \ref{sec:exp_ext}.
We also showcase even stronger performance gains of our approach when using the negative log-likelihood for comparison instead of the classification accuracy.

\subsection{Simulations of Toy Tasks for Common Classifiers}
\label{sec:more_sims}

As already mentioned in Section \ref{sec:exp}, we compare a neural network with the classifiers k-nearest neighbors, Support Vector Machine, Decision Tree,
Random Forest, XGBoost, Naive Bayes, and a neural network.
The neural network is implemented via PyTorch \citep{NEURIPS2019_9015}.
It has a single hidden layer and 100 nodes. 
It is trained with the log-likelihood as criterion, the Adam optimizer provided by PyTorch, and early stopping (we split off 30\% of the training set).
For the other classifiers, we use the implementations from Scikit-Learn \citep{scikit-learn}.
The hyperparameters are the following.
The k-nearest neighbors uses $k=5$, the SVM classifier uses $C=1$ and $\gamma = 2$, the gaussian process classifier uses the RBF kernel.
For Random Forests and XGBoost, we use an ensemble size of ten.
The naive bayes classifer uses a gaussian assumption.
All the other hyperparameters are defaults by Scikit-Learn.

\begin{figure*}[t]
\vspace{.3in}
\centerline{\includegraphics[width=\linewidth]{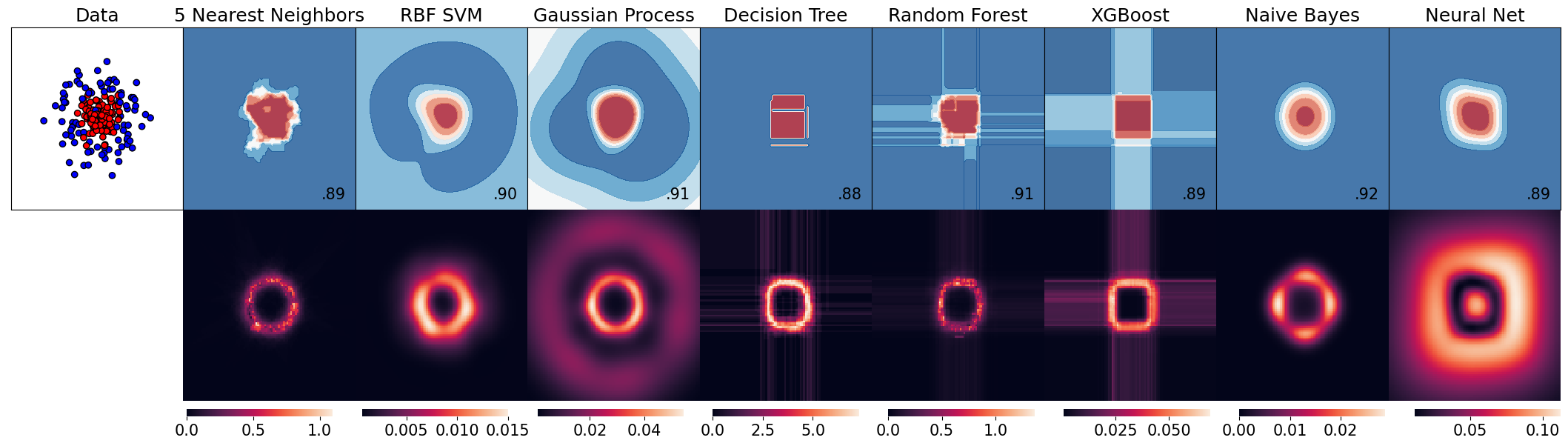}}
\vspace{.3in}
\caption{
\textbf{Top}: Several classifiers are trained on a simulated circular task and their predictions are shown around the input space. The number in the bottom right corner is the accuracy.
\textbf{Bottom}: The Bregman Information of these classifiers is estimated based on several training runs for the identical space.
}
\label{fig:BI_toy_1}
\end{figure*}

\begin{figure*}[t]
\vspace{.3in}
\centerline{\includegraphics[width=\linewidth]{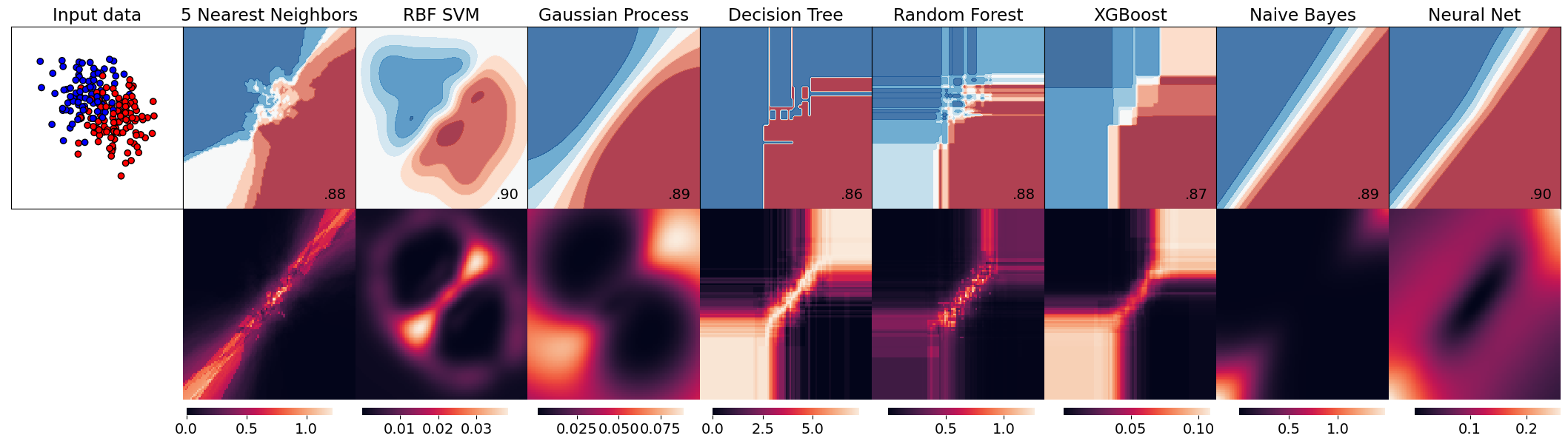}}
\vspace{.3in}
\caption{
\textbf{Top}: Several classifiers are trained on a simulated linear task and their predictions are shown around the input space. The number in the bottom right corner is the accuracy.
\textbf{Bottom}: The Bregman Information of these classifiers is estimated based on several training runs for the identical space.
}
\label{fig:BI_toy_2}
\end{figure*}

%The code for simulating the tasks and running the experiments is provided in the supplementary material.
The simulated data sets have 300 train instances, and 200 test instances.
We construct two more toy tasks: One of circular shape with closed decision boundary, the other of linear shape.
The results are depicted in Figure \ref{fig:BI_toy_1} and \ref{fig:BI_toy_2}.
The Bregman Information in the main paper and these figures are based on 64 training set samples.
As can be seen, %the neural network shows its highest stability along the decision boundary as long as we are in-domain.
%This is a stark contrast compared to the other classifiers.
SVMs and Gaussian Processes can indicate where the training distribution ends, while the BI of other classifiers such as KNN only identifies the direction of the decision boundary.
This might make SVMs and Gaussian Processes a potential tool for out-of-domain detection for low-dimensional data.

Surprisingly, the neural network shows its lowest uncertainty around the decision boundary.
Even in areas, where are sufficiently enough data samples of a class, the neural network shows uncertainty where other classifiers do not.
We hypothesis a possible reason for this might be that neural networks are optimized via gradient descent and the log-likelihood, which requires anchor points of both classes for a stable convergence.
Around areas with instances of only a single class, gradient descent is missing an anchor and does not 'know' how far to fit the model towards this class.
At first, this might discourage using Bregman Information for out-of-domain detection at high-dimensional tasks, such as image data, fitted with a neural network.
But, the traversing of the decision boundary from in-domain to out-of-domain still gives gives the highest Bregman Information of the neural network in our simulations.
Consequently, in the high-dimensional setting, if most data instances lie on the decision boundary and the decision boundary is 'open' in a variety of directions, we might still receive sufficient indication of in- and out-of-domain areas in the input space.
Our results in Section \ref{sec:exp} and Section \ref{sec:exp_ext} support this hypothesis.

Similar to Figure \ref{fig:BI_approx}, we provide the same approximations and MC Dropout for additional toy tasks in Figure \ref{fig:BI_approx_ext}.
In all cases, for the Deep Ensemble \citep{lakshminarayanan2017simple} we use 64 models, for MC Dropout \citep{gal2016dropout} an ensemble size of 5000, and for the 'real' BI we use 64 training set samples.
Again, the results in Section \ref{sec:exp} and Section \ref{sec:exp_ext} support that the low-dimensional findings hold to some degree for real-world image data.

\begin{figure*}
\vskip 0.2in
\centering
    \begin{subfigure}{.7\textwidth}
    \centering
    \includegraphics[width=\columnwidth]{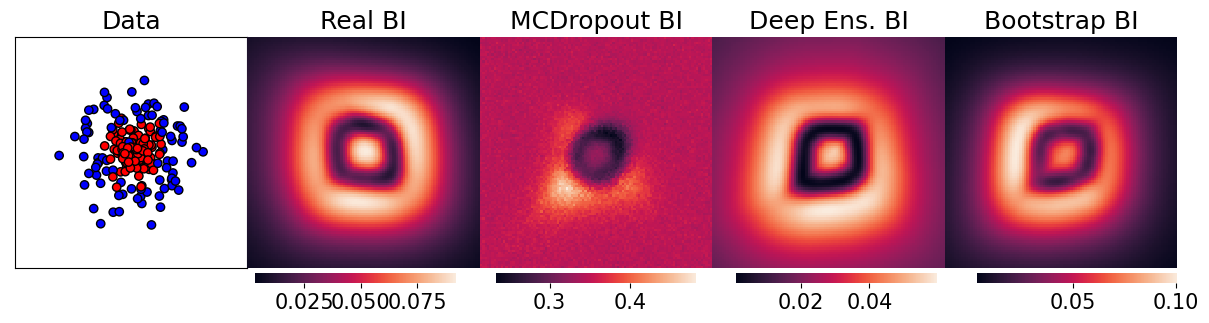}
    \caption{Circular task}
    \end{subfigure} \\ % 
    \begin{subfigure}{.7\textwidth}
    \centering
    \includegraphics[width=\columnwidth]{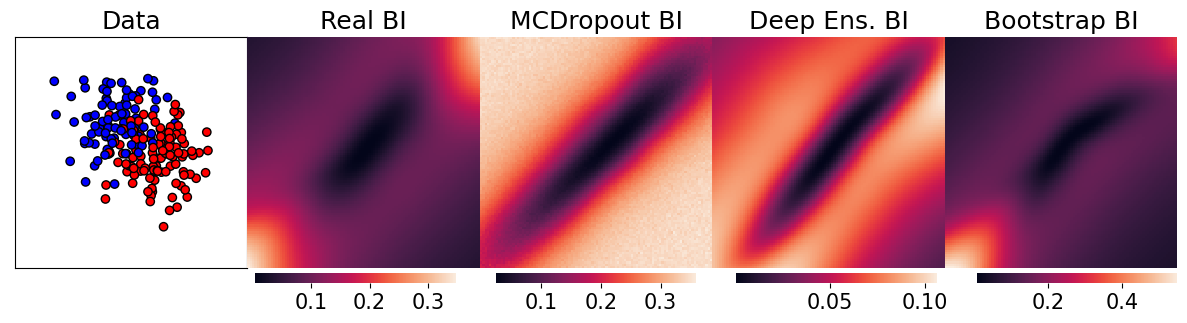}
    \caption{Linear task}
    \end{subfigure}
\caption{Different approximations of the Bregman Information for a neural network. 'Real BI' refers to the approximation via training set samples from the real distribution. The other approximations are only with respect to a single training set.}
\vskip -0.2in
\label{fig:BI_approx_ext}
\end{figure*}

\subsection{Additional Out-of-Distribution Results on CIFAR-10 and ImageNet, and Further Details}
\label{sec:exp_ext}

In this section, we provide further results for uncertainty thresholds in the out-of-distribution setting of CIFAR-10 \citep{krizhevsky2009learning} and ImageNet \citep{krizhevsky2009learning}.
We will also discuss further experiment details.

\paragraph{Comparisons via negative log-likelihood instead of accuracy}
The log-likelihood is a proper score and as such a measure of predictive uncertainty.
It captures the correctness of a predicted probability instead of only the correctness of the predicted class, like accuracy.
Consequently, the log-likelihood indicates how trustworthy confidence scores are.
We conduct similar experiments as in the main paper but replace the accuracy with the log-likelihood.
The results can be seen in Figure \ref{fig:cif_duo_nll_comp}.
The performance improvement of Bregman Information with Deep Ensembles for out-of-domain instances is substantial compared to Confidence scores.

\begin{figure*}
\vskip 0.2in
\centering
\begin{subfigure}{.5\textwidth}
    \centering
    \includegraphics[width=\columnwidth]{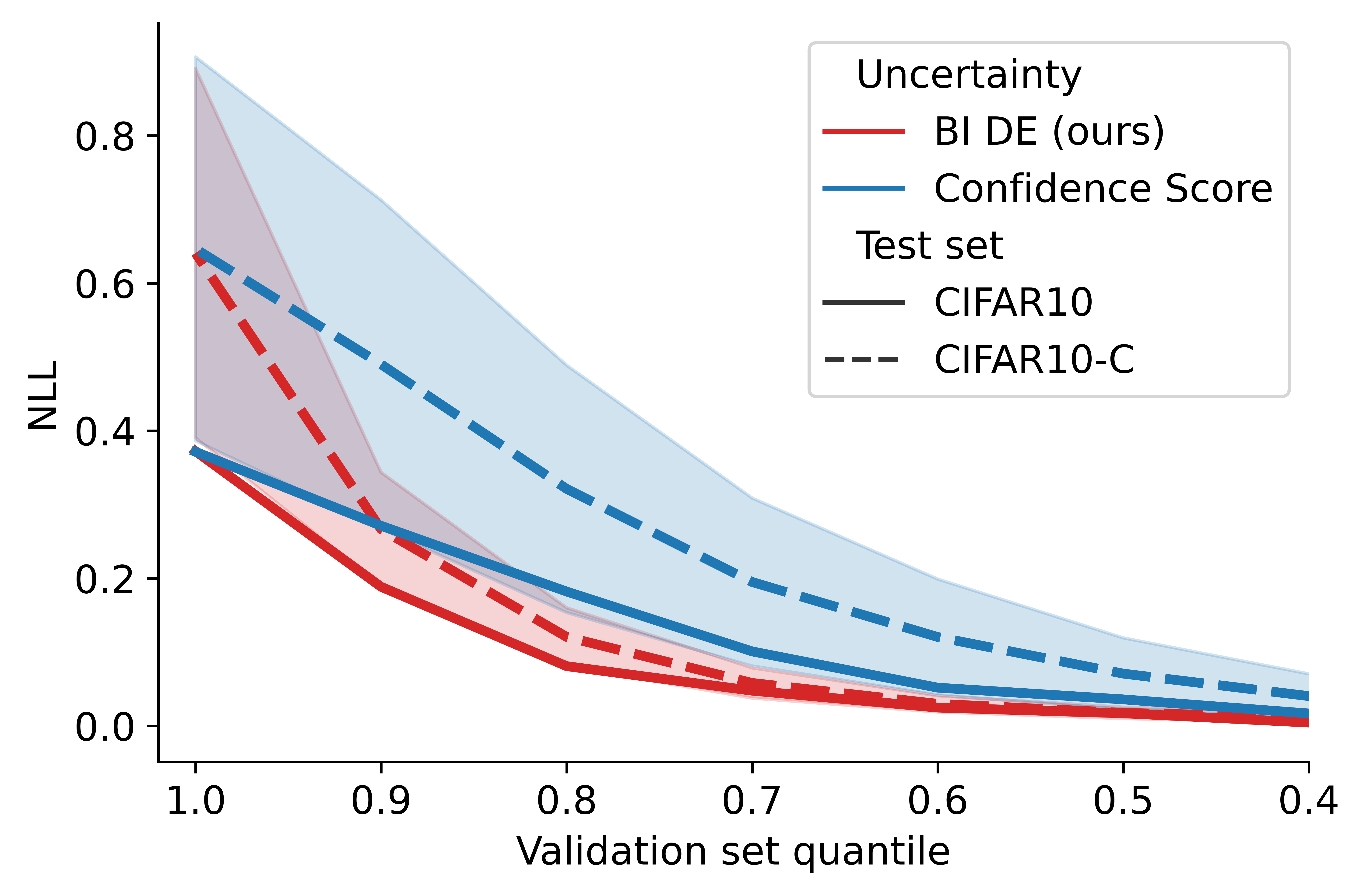}
    \caption{Combining all severities.}
    \end{subfigure}%
    \begin{subfigure}{.5\textwidth}
    \centering
    \includegraphics[width=\columnwidth]{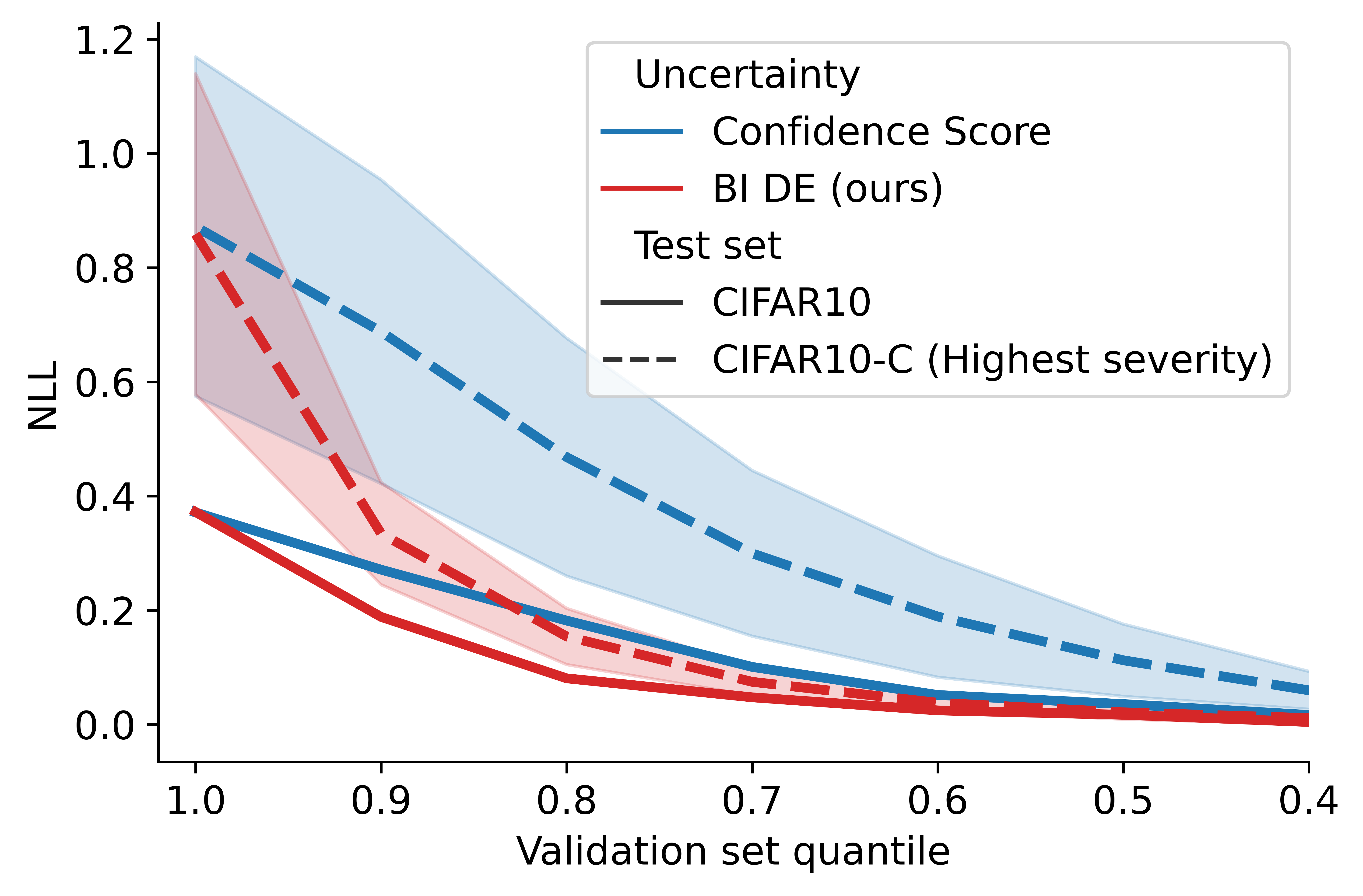}
    \caption{Only highest severity.}
    \end{subfigure}%
\caption{Negative Log-Likelihood after discarding test instances with high
levels of uncertainty for CIFAR-10 and CIFAR-10-C. Fewer
samples have to be discarded to reach better NLL when
using the Bregman Information as uncertainty measure.}
\vskip -0.2in
\label{fig:cif_duo_nll_comp}
\end{figure*}

\paragraph{Datasets}

To compare in-domain with out-of-domain performance, we use corrupted versions of the test sets introduced in  \citep{hendrycks2019robustness}.
The test sets CIFAR-10-C and ImageNet-C have 5 different severities for 20 different corruptions: Brightness, fog, glass blur, pixelate, spatter, contrast, frost, impulse noise, saturate, speckle noise, defocus blur, gaussian blur, jpeg compression, shot noise, zoom blur, elastic transform, gaussian noise, motion blur, and snow.
For CIFAR-C-10, we have 10000 test instances per corruption per severity.
We have to remove 10000 instances in each ImageNet-C corruption severity, which are corruptions of our validation set, leaving us 40000 test instances per corruption per severity.

\begin{figure*}
\vskip 0.2in
\centering
    \begin{subfigure}{.5\textwidth}
    \centering
    \includegraphics[width=\columnwidth]{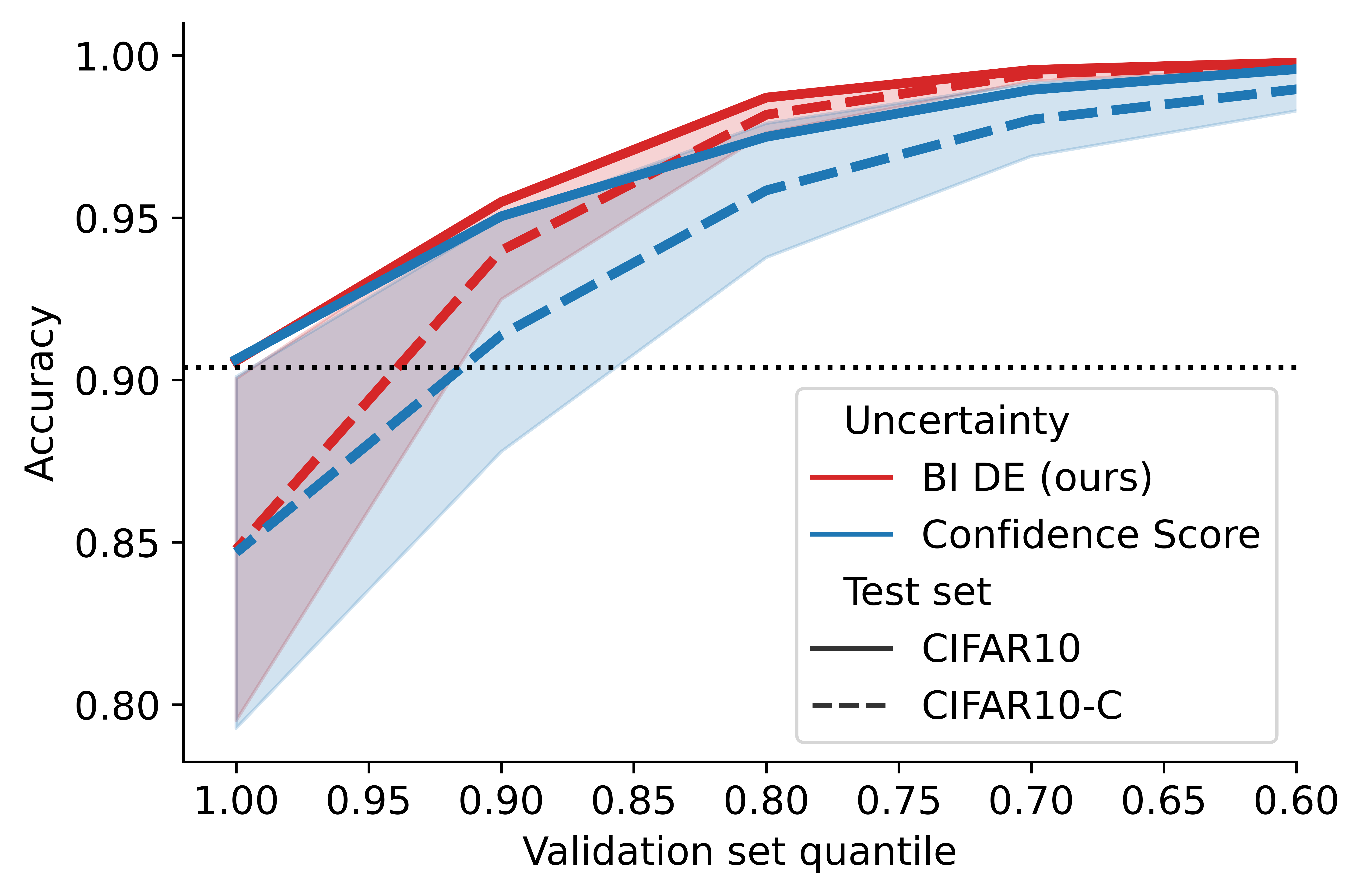}
    \caption{Deep Ensembles (all Severities)}
    \end{subfigure}%
    \begin{subfigure}{.5\textwidth}
    \centering
    \includegraphics[width=\columnwidth]{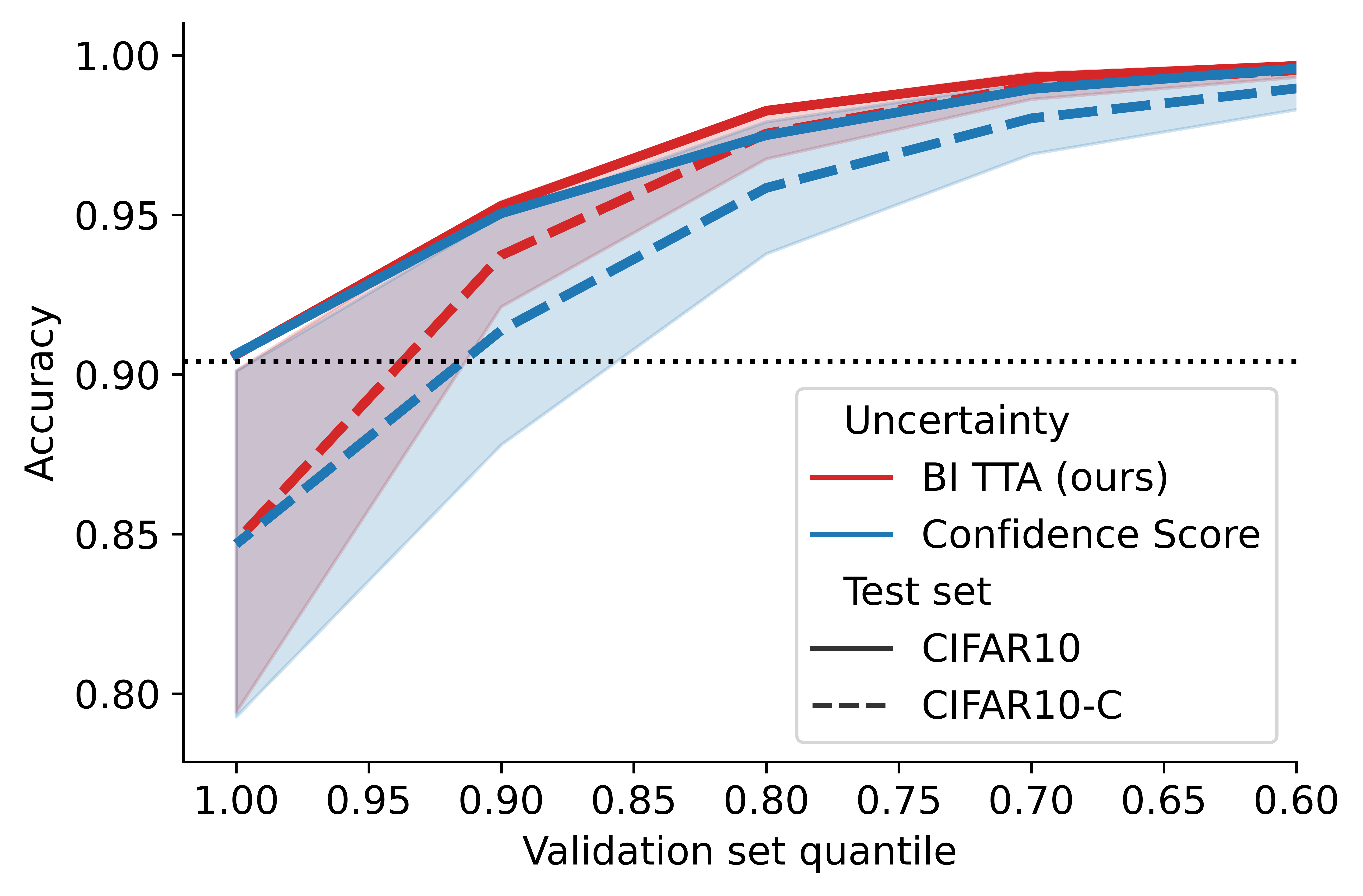}
    \caption{Test-Time Augmentation (all Severities)}
    \end{subfigure} \\
    \begin{subfigure}{.5\textwidth}
    \centering
    \includegraphics[width=\columnwidth]{figures/resnet_Cifar10-C_quantile_hueUncertainty_sev5_uncDE_Accuracy.png}
    \caption{Deep Ensembles (only Severity 5)}
    \end{subfigure}%
    \begin{subfigure}{.5\textwidth}
    \centering
    \includegraphics[width=\columnwidth]{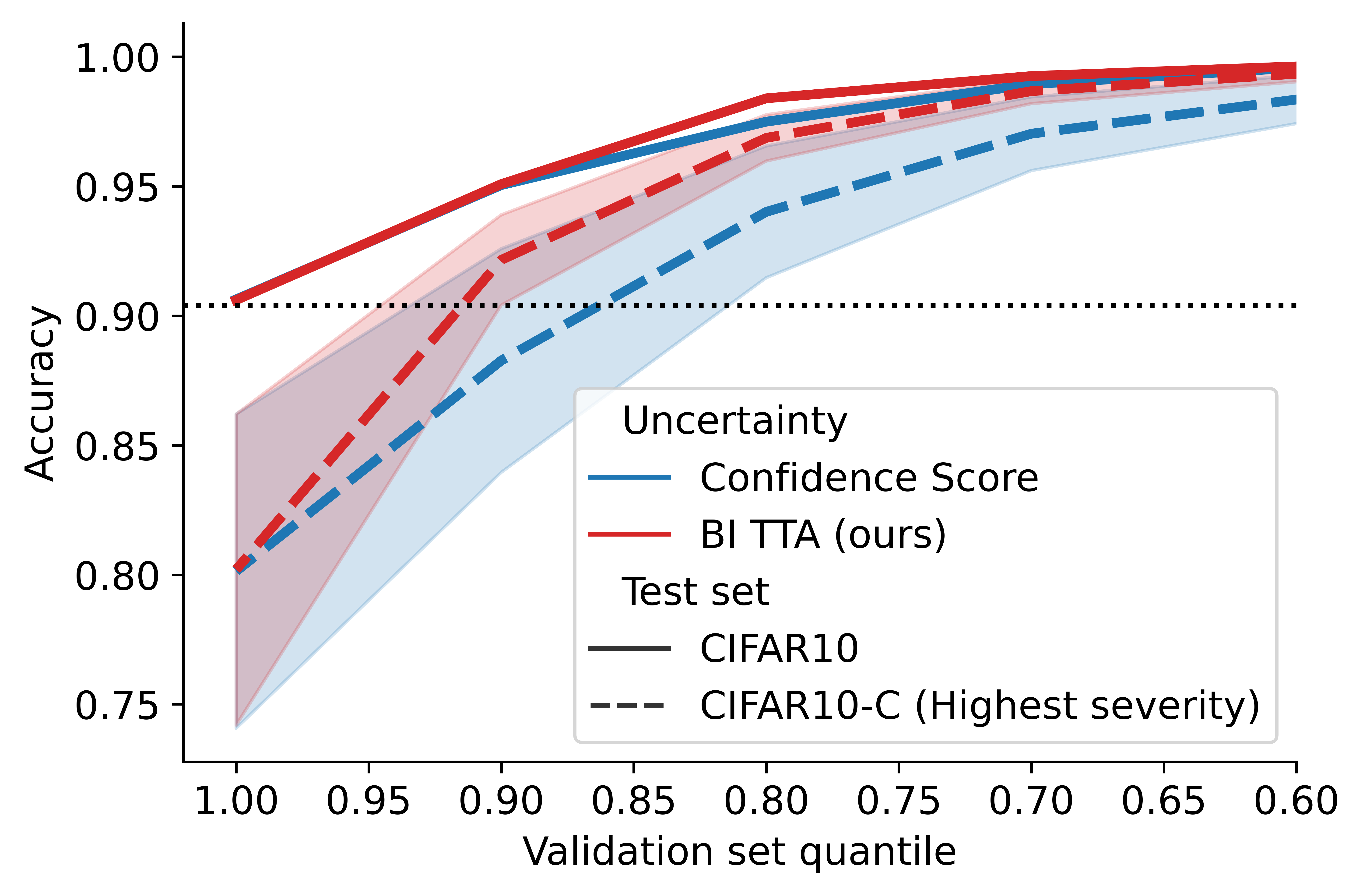}
    \caption{Test-Time Augmentation (only Severity 5)}
    \end{subfigure}%
    
\caption{Accuracy after discarding test instances with high
levels of uncertainty for CIFAR-10 and CIFAR-10-C. Fewer
samples have to be discarded to reach better accuracy when
using the Bregman Information as uncertainty measure.}
\vskip -0.2in
\label{fig:all_cor_comp}
\end{figure*}

\paragraph{Models and ensembles}

For classification on CIFAR-10, we use a ResNet20 trained with Adam and early stopping \citep{he2016deep} based on the PyTorch framework.
We train the Deep Ensemble of size 10 by training the same architecture with different weight initializations.

For ImageNet, we use ResNet50 models downloaded from \citep{ashukha2020pitfalls}.\footnote{https://github.com/SamsungLabs/pytorch-ensembles}
We also use an Deep Ensemble of size 10.

Further, we use the ensembling technique Test-Time Augmentation \citep{wang2019aleatoric}.
The augmentations are Random Crop, Random Flip, and we use an ensemble size of 20.

Figure \ref{fig:all_cor_comp} shows similar results for TTA as Figure \ref{fig:ood_cif10} in the main paper.
Further, our approach still dominates when we include all corruption severities.
Note that the deviation bounds are smaller, which indicates that BI is more robust for different types of corruptions.

\paragraph{Uncertainty threshold algorithm for Confidence scores}
We also provide the Algorithm \ref{alg:BI} adjusted to Confidence scores.
It is described in Algorithm \ref{alg:Conf}.
The only difference is that we are not using an ensemble anymore and we flip the threshold, since higher confidence means less uncertainty, while higher BI means lower uncertainty.

\begin{algorithm}
\caption{Classifying with uncertainty threshold via Confidence scores.
}
\label{alg:Conf}
\begin{algorithmic}
\Require Validation set $\mathcal{D}$, model, $q \in \left[0, 1 \right]$, test instance $x^\prime$
%\Ensure $y = x^n$
\State ConfScores $\gets$ [$\max_i$ model$(x)_i$ for $x \in \mathcal{D}$]
\Comment{Highest predicted probability for each instance}
\State threshold $ \gets $ quantile(ConfScores, q)
\If{$\max_i$ model$(x^\prime)_i <$ threshold}
    \State label as OOD \Comment{Warning in real-world application}
\Else
    \State return model$(x^\prime)$
\EndIf
\end{algorithmic}
\end{algorithm}

\end{document}